%% file: Pethick_AccSODA_2026.tex
\documentclass{article}

\usepackage[pagebackref]{hyperref}
\usepackage[dvipsnames]{xcolor}
 \hypersetup{
    colorlinks=true,
    linkcolor=MidnightBlue,
    citecolor=MidnightBlue
 }
\renewcommand*{\backref}[1]{}
\renewcommand*{\backrefalt}[4]{%
  \ifcase #1
  \or Cited on page~#2.%
  \else Cited on pages~#2.%
  \fi
}

\PassOptionsToPackage{square}{natbib}

\usepackage[preprint]{neurips_2026}

\usepackage[utf8]{inputenc} %
\usepackage[T1]{fontenc}    %
\usepackage{url}            %
\usepackage{booktabs}       %
\usepackage{amsfonts}       %
\usepackage{nicefrac}       %
\usepackage{microtype}      %

\usepackage{amsmath}
\usepackage{amssymb}
\usepackage{mathtools}
\usepackage{amsthm}

\input{preamble}

\title{When to use what Schatten-$p$ norm in deep learning?}

\author{%
  Thomas Pethick \\
  \texttt{tmpethick@gmail.com}
}

\begin{document}

\maketitle

\begin{abstract}
Schatten-$\infty$ based optimizers such as Muon have shown promising empirical performance, but there remains seemingly conflicting observations regarding whether they are beneficial.
We resolve this conflict by showing that the conclusion is regime dependent.
Even when the objective is smooth in the Schatten-$\infty$ geometry, smaller Schatten-$p$ geometries can be optimal, specifically in the low-dimensional regime, which we show includes Chinchilla scaling.
This conclusion follows from a new noise-robust acceleration result for the SODA framework for $p>2$.
The same analysis explains why Muon-like methods do not require warmup, why they naturally favor large batches, and yields a batch size scaling rule for arbitrary $p$.
\end{abstract}

\input{body}

\end{document}

%% file: preamble.tex
\usepackage{enumitem}

\usepackage[capitalize,noabbrev]{cleveref}

\usepackage{crossreftools}
\usepackage{xargs} 
\usepackage{xstring}
\usepackage{etoolbox}

\usepackage{aliascnt}

\newtheorem{thm}{Theorem}[section]
\newaliascnt{cor}{thm}
\newtheorem{cor}[cor]{Corollary}
\aliascntresetthe{cor}
\newaliascnt{lemma}{thm}
\newtheorem{lemma}[lemma]{Lemma}
\aliascntresetthe{lemma}
\newaliascnt{fact}{thm}

\aliascntresetthe{fact}
\newaliascnt{prop}{thm}

\aliascntresetthe{prop}

\newtheorem{definition}{Definition}
\newtheorem{assumption}{Assumption}

\theoremstyle{remark}
\newtheorem{remark}{Remark}

\crefname{thm}{theorem}{theorems}
\crefname{definition}{definition}{definitions}
\Crefname{definition}{Definition}{Definitions}
\crefname{assumption}{assumption}{assumptions}
\Crefname{assumption}{Assumption}{Assumptions}
\crefname{cor}{corollary}{corollaries}
\Crefname{cor}{Corollary}{Corollaries}
\crefname{prop}{proposition}{propositions}
\Crefname{prop}{Proposition}{Propositions}
\crefname{lemma}{lemma}{lemmas}
\Crefname{lemma}{Lemma}{Lemmas}
\crefname{fact}{fact}{facts}
\crefname{remark}{remark}{remarks}
\Crefname{remark}{Remark}{Remarks}

\usepackage{mdframed}
\newmdtheoremenv{algo}{Algorithm}
\newmdtheoremenv{procedure}{Decision process}

\newlist{assnum}{enumerate}{1} %
\setlist[assnum]{label=(\roman*), ref=\theassumption(\roman*)}
\crefalias{assnumi}{assumption} 

\newlist{lemnum}{enumerate}{1} %
\setlist[lemnum]{label=(\roman*), ref=\thelemma(\roman*)}
\crefalias{lemnumi}{lemma} 

\newlist{thmnum}{enumerate}{1} %
\setlist[thmnum]{label=(\roman*), ref=\thethm(\roman*)}
\crefalias{thmnumi}{thm} 

\newlist{cornum}{enumerate}{1} %
\setlist[cornum]{label=(\roman*), ref=\thecor(\roman*)}
\crefalias{cornumi}{cor} 

\newlist{definitionnum}{enumerate}{1} %
\setlist[definitionnum]{label=(\roman*), ref=\thedefinition(\roman*)}
\crefalias{definitionnumi}{definition} 

\newlist{propnum}{enumerate}{1} %
\setlist[propnum]{label=(\roman*), ref=\theproposition(\roman*)}
\crefalias{propnumi}{prop}

\newlist{examplenum}{enumerate}{1} %
\setlist[examplenum]{label=(\roman*), ref=\theexample(\roman*)}
\crefalias{examplenumi}{example} 

\newcommand\numberthis{\addtocounter{equation}{1}\tag{\theequation}}		%

\usepackage{dsfont}
\usepackage{nicefrac}

\DeclareMathOperator*{\argmin}{arg\,min}

\newcommand{\R}{\mathbb{R}}

\newcommand{\lmo}{\operatorname{lmo}}

\newcommand{\norm}[1]{\left\Vert{#1}\right\Vert}

\newcommand{\abs}[1]{\left\vert{#1}\right\vert}

\usepackage{tcolorbox}
\newtcolorbox{defbox}{colback=black!5!white,colframe=black!75!black}
\newtcolorbox{asmbox}{colback=black!5!white,colframe=black!75!black}
\newtcolorbox{thmbox}{colback=red!5!white,colframe=red!75!black}

\usepackage{braket}

\newcommandx{\QC}[2][1={},2={}]{\ifstrempty{#1}{Q#2}{Q_{#1}\ifstrempty{#2}{}{(#2)}}}
\newcommandx{\PC}[2][1={},2={}]{\ifstrempty{#1}{P#2}{P_{#1}\ifstrempty{#2}{}{(#2)}}}
\newcommandx{\HC}[2][1={},2={}]{\ifstrempty{#1}{H#2}{H_{#1}\ifstrempty{#2}{}{(#2)}}}
\newcommandx{\MC}[2][1={},2={}]{\ifstrempty{#1}{M#2}{M_{#1}\ifstrempty{#2}{}{(#2)}}}

\newcommandx{\EF}[2][1={k},2={}]{\mathbb E\ifstrempty{#1}{}{_{#1}}#2}

\usepackage{algorithm}
\usepackage[noend]{algpseudocode}

\newcommand{\equationbox}[1]{%
\par\noindent\makebox[\linewidth][c]{%
\fbox{%
\begin{minipage}{\dimexpr\linewidth-2\fboxsep-2\fboxrule\relax}
#1
\end{minipage}%
}}%
\par}

\usepackage{adjustbox}

\usepackage{array,multirow,hhline,graphicx}
\usepackage{subcaption}
\usepackage{wrapfig}
\usepackage{colortbl}

\usepackage{pifont}
\newcommand{\cmark}{\ding{51}}
\newcommand{\xmark}{\ding{55}}

\usepackage[para,online,flushleft]{threeparttable}

%% file: body.tex
\section{Introduction}\label{sec:intro}
\input{Sections/Introduction}

\section{Method}\label{sec:method}
\input{Sections/Method}

\section{Analysis}\label{sec:analysis}
\input{Sections/AnalysisV2}

\section{Related work}\label{sec:related-work}
\input{Sections/RelatedWork}

\section{Conclusion}\label{sec:conclusion}
\input{Sections/Conclusion}

\section{Acknowledgments}\label{sec:acknowledgments}
\input{Sections/Acknowledgments}

\bibliographystyle{plainnat-hiddenlinks}
\bibliography{refs2.bib,lions-master.bib}

\clearpage
\appendix
\section{Proofs}\label{app:proofs}
\input{Sections/AnalysisV2Appendix}

%% file: Sections/Introduction.tex
Schatten-$\infty$ based optimizers such as Muon \citep{jordan2024muon} have seen substantial adoption in the machine learning community.
However, there remain seemingly conflicting observations around when, and whether, Muon is beneficial.

One prevailing consensus is that the batch size needs to be taken large to see the benefits of the Schatten-$\infty$ norm.
This was the case in the original spectral methods \citep{carlson2016stochastic,carlson2015preconditioned}, later in the CIFAR-10 speedrun setting where Muon was developed \citep{jordan202494}, and further documented in \citet{pethick2025trainingdeeplearningmodels,shah2025practical}.
It also seems that Muon is particularly beneficial for short training runs, such as speedruns on modded-nanogpt \citep{modded_nanogpt_2024}.

However, for large-scale training runs, where the token budget is also much larger, the verdict is less clear as shown in large-scale benchmark studies such as \citet{wen2025fantastic,semenov2025benchmarking}.
Even SGD has been shown to outperform common training methods like Adam, specifically in the extremely low batch size regime \citep{sreckovic2025your,marek2026small}.

Additionally, HTMuon \citep{pang2026htmuon} and Soft-Muon \citep{nilin2024contramuon}, which instead use finite Schatten-$p$ norms, have respectively led to a speedup on Llama pre-training and to the leading optimizer-track result on modded-nanogpt at the time of writing.
Similarly, \citet{shumaylov2026muon} dispute the need for exclusively using the Schatten-$\infty$ norm.

Essentially, the mechanism behind Muon's success, and the regime in which it should be expected to be beneficial, is not fully understood.
This raises the following more general question:
\begin{center}
\emph{What Schatten-$p$ norm should the optimizer use, and when?}
\end{center}
To answer this question, we build on the SODA framework of \citet{pethick2026optimistic}, which combines optimistic dual averaging \citep{rakhlin2013online} with the schedule-free optimization framework of \citet{defazio2024road}.
This lets us capture modern optimizer components precisely, including dualization, Nesterov momentum and weight decay.

From this perspective, we identify the possible source of the confusion.
The choice of $p$ is simultaneously tied to both the heavy-tailedness of the stochastic gradient and the geometry of the problem, \emph{and} it affects the degree to which the method can accelerate.
We find that \emph{even when the objective is smooth in the Schatten-$\infty$ geometry, it can be beneficial to use an optimizer with a smaller Schatten-$p$ norm in the low-dimensional, or ``overtraining,'' regime.} Thus, conclusions will be regime dependent.

Throughout, we use ``$p$-norm'' to refer to either a vector $\ell_p$ norm or a Schatten-$p$ norm, depending on the ambient space.
The first challenge we face is that existing analysis of SODA does not allow us to capture $p$-norm geometries for $p>2$.
A major contribution is thus to extend the analysis of SODA to these $p$-norm geometries by considering $p$-uniformly convex reference functions.

\textbf{Contributions.} Concretely, we make the following contributions:
\begin{enumerate}[label=(\roman*)]
    \item We prove that \ref{eq:SODA} achieves accelerated rates for optimization over $p$-norm geometries while remaining robust to stochastic noise.
    In particular, we obtain the \emph{dimension-free} rate
    \[
        O\!\left(
        \frac{L_p R_p^2}{n^{1+2/p}}
        +
        \frac{\sigma_q R_p}{n^{1/p}}
        \right),
    \]
    revealing a connection between the choice of geometry $p$ and the heaviness of the noise tails through the dual exponent $q=p/(p-1)$.
    This appears to be the first noise-robust acceleration result for arbitrary $p$-norm geometries.
    The result extends to the H\"older-smooth case and interestingly yields a $p$-dependent stepsize schedule.
    \looseness=-1

    \item We derive stochastic gradient oracle complexity bounds under possibly lighter-tailed noise by using large batch sizes.
    In the extreme case of bounded variance, the result shows that after $N$ oracle calls, Euclidean methods recover the classical $O(N^{-1/2})$ stochastic error rate, whereas $p\rightarrow \infty$ gives $O(N^{-1/3})$.
    In the process, we develop a norm-dependent batch size scaling rule that increases with $p$ as $B_p\asymp N^{2(p-1)/(3p-2)}$.

    \item We interpret these results and show that smoothness in the $\infty$-norm does \emph{not} always imply that the $\infty$-geometry is optimal for the algorithm.
    Lower values of $p$ can be preferable both in deterministic settings, due to stronger acceleration, and under heavy-tailed noise, due to better alignment with the noise assumption.
    This advantage appears in the low-dimensional regime, where the effective dimension $d$ is smaller than the oracle budget $N$, i.e., $d\lesssim N$.
    This regime dependency suggests a switching strategy of $p$ during training, which we discuss.

    \item We connect the low-dimensional regime with modern scaling rules, illustrating that common scaling rules in fact place training well within the low-dimensional regime by scaling the token budget $N$ as $d^{3}$ where $d$ is the effective dimension of the model matrices.
    This might explain why Muon is not necessarily favored under typical scaling recipes.
\end{enumerate}

\begin{table*}[!t]
\centering
\caption{Instances of \ref{eq:SODA} obtained by various choices
of $p$-norm. For
Schatten norms, write $Y=U\operatorname{diag}(s)V^\top$. Optimism means
$\bar\alpha_k\ne0$ with Nesterov momentum corresponding to $\bar\alpha_k=\alpha_k$.}
\label{tbl:soda}
\begin{threeparttable}
\small
\setlength{\tabcolsep}{3pt}
\begin{tabular*}{\textwidth}{@{\extracolsep{\fill}}>{\raggedright\arraybackslash}p{0.16\textwidth}
>{\raggedright\arraybackslash}p{0.22\textwidth}
>{\centering\arraybackslash}p{0.105\textwidth}
>{\raggedright\arraybackslash}p{0.13\textwidth}
>{\raggedright\arraybackslash}p{0.265\textwidth}@{}}
\toprule
\textbf{Method} & \textbf{Reference} & \textbf{Optimism} &
\textbf{Norm} &
\textbf{Dual map ($\partial h^*/\nabla h^*$)} \\
\midrule
Lion &
\citep{chen2023symbolic} &
\cmark &
$\ell_\infty$ &
$\operatorname{sign}(u)$ \\
 &
-- &
&
$\ell_p$ &
$\operatorname{sign}(u)\odot\abs{u}^{q-1}$ \\
\midrule
Muon &
\citep{jordan2024muon} &
\cmark &
Schatten-$\infty$ &
$UV^\top$ \\
HTMuon\tnote{1} &
\citep{pang2026htmuon} &
\cmark &
Schatten-$p$ &
$U\operatorname{diag}(s^{q-1})V^\top$ \\
\midrule
Scion ({\footnotesize RowNorm}) &
\citep{pethick2025trainingdeeplearningmodels} &
\xmark &
$2\to\infty$ &
$Y_i/\norm{Y_i}_2$ \\
 &
-- &
&
mixed $\ell_{2,p}$ &
$\norm{Y_i}_2^{q-2}Y_i$ \\
\bottomrule
\end{tabular*}
\begin{tablenotes}
\item[1] The same finite-Schatten norm is used in Soft-Muon and Freon
\citep{nilin2024contramuon,shumaylov2026muon}.
\end{tablenotes}
\end{threeparttable}
\end{table*}

%% file: Sections/Method.tex
Let $\mathcal X$ be a finite-dimensional normed space with dual $\mathcal X^*$.
We rely on the Fenchel conjugate to map gradient information back to the primal space.
For a proper closed convex reference function $h:\mathcal X\to\R\cup\{\infty\}$, the Fenchel conjugate over $\mathcal X$ is defined as,
\[
h^*(m)=\sup_{x\in\mathcal X}\{\braket{m,x}-h(x)\},
\qquad m\in\mathcal X^*.
\]

We analyze the following algorithm, introduced in \citet{pethick2026optimistic}, which generalizes Optimistic Dual Averaging (ODA) by introducing a primal extrapolation sequence ($y^k$) from \citet{tseng2008accelerated,lan2012optimal,defazio2024road}:
\begin{equation*}\label{eq:SODA}
\tag{SODA}
\begin{split}
m^{k+1} &= (1-\alpha_k) m^k + \alpha_k \nabla f(y^k,\xi_k), \\
\bar{m}^{k+1} &= (1-\bar{\alpha}_k) m^{k+1}
    + \bar{\alpha}_k \nabla f(y^k, \xi_k), \\
z^{k+1} &\in \partial h^*(-\gamma_k\bar{m}^{k+1})
    = \argmin_{x \in \mathcal X}
    \gamma_k\braket{\bar{m}^{k+1},x} + h(x), \\
x^{k+1} &= (1-\lambda_k)x^k + \lambda_k z^{k+1}, \\
y^{k+1} &= (1-\bar\lambda_k)x^{k+1}+\bar\lambda_k z^{k+1}.
\end{split}
\end{equation*}
Here $\alpha_k,\bar\alpha_k,\lambda_k,\bar\lambda_k\in[0,1]$ and $\gamma_k>0$.
We initialize $m^0=0$ and $x^0=y^0=z^0\in\partial h^*(0)$.
For $\bar\lambda_k=0$ we obtain an optimistic version of the Double Averaging \citep{nesterov2015quasi}.

What is particularly attractive about this algorithmic template in the context of deep learning is that it allows us to capture commonly used techniques such as Nesterov momentum ($\bar m^{k+1}$), weight decay (parameterized by $\lambda_k$ when $\bar\lambda_k=0$) and the use of different geometries ($\partial h^*$), all in a way where their impact can be understood theoretically.

\subsection{Norm choices}

Previous analyses of SODA were based on strongly convex reference functions, which naturally cover $p$-norm geometries for $p\in(1,2]$.
We instead use uniform convexity, relying on the fact that $h(x)=\frac{1}{p}\|x\|_p^p$ is $p$-uniformly convex (see \Cref{def:uniform-convexity}).
At the endpoint $p=\infty$, this regularizer becomes a hard $\infty$-norm constraint, thus recovering the Schatten-$\infty$ linear minimization oracle ($\lmo$) used in e.g., Muon \citep{jordan2024muon} and Scion \citep{pethick2025trainingdeeplearningmodels}.

\paragraph{Vector $\ell_p$ geometry}
For vector variables $x\in\mathbb R^d$, choose the fixed reference function centered at the initial point $z^0$,
\[
h(x)=\tfrac1p\norm{x-z^0}_{\ell_p}^p,
\qquad p>2,
\qquad q=\tfrac{p}{p-1}.
\]
Then $h^*(u)=\braket{u,z^0}+\tfrac1q\norm{u}_{\ell_q}^q$, and the step in \ref{eq:SODA} is explicit:
\[
z^{k+1}
=
\nabla h^*(-\gamma_k\bar m^{k+1})
=
z^0
-\gamma_k^{q-1}
\operatorname{sign}(\bar m^{k+1})
\odot \abs{\bar m^{k+1}}^{q-1},
\]
where the sign, absolute value, power, and product are taken componentwise.
When $z^0=0$, this dual map interpolates between the identity map, $\nabla h^*(g)=g$, at the Euclidean endpoint $p=2$ and the sign map at $p= \infty$.
This is different from, for example, the linear minimization oracle in Frank-Wolfe, which interpolates between \emph{normalized} gradients and the sign map.

\paragraph{Schatten-$p$ geometry}
For matrix variables $X\in\mathbb R^{d_1\times d_2}$, choose instead
\[
h(X)=\tfrac1p\norm{X-Z^0}_{S_p}^p
=\tfrac1p \textstyle \sum_i \sigma_i(X-Z^0)^p,
\qquad p>2,
\qquad q=\tfrac{p}{p-1},
\]
where $\norm{\cdot}_{S_p}$ is the Schatten-$p$ norm.
Then $h^*(Y)=\braket{Y,Z^0}+\tfrac1q\norm{Y}_{S_q}^q$.
If $\bar M^{k+1}=U\operatorname{diag}(s)V^\top$ is a singular value decomposition (SVD), then
\[
Z^{k+1}
=
\nabla h^*(-\gamma_k\bar M^{k+1})
=
Z^0
-\gamma_k^{q-1}U\operatorname{diag}(s^{q-1})V^\top .
\]

This is the dualization used in HTMuon \citep{pang2026htmuon}, Freon \citep{shumaylov2026muon} and Soft-Muon \citep{nilin2024contramuon}, with the only difference being that $Z^0$ can be taken different from zero and the stepsize $\gamma_k$ is exponentiated as $q-1$.
The SVD can be computed efficiently using QDWH \citep{nakatsukasa2010optimizing,nakatsukasa2016computing} used in Freon or (warmstarted) power iteration as used in PowerSGD and Dion \citep{vogels2019powersgd,ahn2025dion}.

\paragraph{Spectral geometry}
At the endpoint $p=\infty$, the update rule becomes the linear minimization oracle over the Schatten-$\infty$/spectral norm ball centered around $Z^0$,
\[
Z^{k+1}=Z^0-UV^\top,
\]
for which \ref{eq:SODA} exactly recovers Muon \citep{jordan2024muon} with Nesterov momentum and weight decay by taking $\bar\lambda_k=0$.

%% file: Sections/AnalysisV2.tex
We derive convergence guarantees for \ref{eq:SODA} under uniform convexity in order to generalize to $p$-norm for $p>2$.
The proof is based on an online regret argument, where we let $g^k$ denote an arbitrary gradient-feedback sequence.
In the stochastic optimization setting of \ref{eq:SODA}, we take $g^k := \nabla f(y^k,\xi_k)$.

\begin{definition}[$p$-uniform convexity]
\label{def:uniform-convexity}
For $p\ge2$, a proper closed convex function $h$ is $p$-uniformly convex with constant $\mu>0$ with respect to $\norm{\cdot}$ if
\[
D_h(u,v;s):=h(u)-h(v)-\braket{s,u-v}
\ge \tfrac{\mu}{p}\norm{u-v}^p
\qquad
\text{for all }u,v\in\mathcal X,\ s\in\partial h(v).
\]
The conjugate exponent is denoted by $q=p/(p-1)$.
The case $p=2$ is the usual strong convexity.
\end{definition}
\begin{assumption}[Convex]
\label{ass:convex}
The objective $f$ is convex.
\end{assumption}

\begin{assumption}[$L$-smooth]
\label{ass:smooth}
The function $f$ is $L$-smooth with respect to $\norm{\cdot}$.
\end{assumption}
\vspace{1mm}
\begin{remark}[Meaning of $\norm{\cdot}_p$]
\label{rem:p-norm-convention}
The general statements are formulated for an arbitrary norm $\norm{\cdot}$.
When the corollaries specialize to $\norm{\cdot}_p$, this notation covers both the vector norm $\norm{\cdot}_{\ell_p}$ and the Schatten norm $\norm{\cdot}_{S_p}$, with dual norm $\norm{\cdot}_q$ interpreted in the same ambient space.
We write $L_p$ for the corresponding smoothness constant.
The specialization is valid because, for finite $p\ge2$, $h(u)=\tfrac1p\norm{u-z^0}_p^p$ is $p$-uniformly convex with respect to either geometry with admissible constant $\mu=2^{2-p}$.
See \Cref{app:additional-p-uniform-norms} for other norms.
\end{remark}
\vspace{1mm}

\begin{assumption}[Unbiased]
\label{ass:soda-stochastic}
Let $\mathcal F_k$ be the natural filtration.
The gradients satisfy
\[
\mathbb E[g^k \mid \mathcal F_{k-1}] = \nabla f(y^k).
\]
\end{assumption}

We rely on the following heavy-tailed type noise assumption which arises naturally in the analysis since $p$-uniform convexity of $h$ implies that $h^*$ is $q$-uniformly smooth.
\begin{assumption}[Gradient variation]
\label{ass:noise}
For some $q\in(1,2]$, the gradients satisfy, for $k\ge0$, $\tau_q\ge1$, and $\sigma_q\ge0$,
\[
\mathbb E\!\left[\norm{g^k-g^{k-1}}_*^q\right]^{1/q}
\le
\tau_q
\mathbb E\!\left[
\norm{\nabla f(y^k)-\nabla f(y^{k-1})}_*^q
\right]^{1/q}
+ \sigma_q .
\]
\end{assumption}
\begin{remark}\label{rem:gradvar}
For $k\ge1$, \Cref{ass:noise} is satisfied under the heavy-tailed noise assumption,
\[
\mathbb E[\norm{\nabla f(y^k,\xi_k)-\nabla f(y^k)}_*^q\mid \mathcal F_{k-1}]
\le \bar\sigma_q^q,
\]
with $\tau_q=1$ and $\sigma_q=2\bar\sigma_q$ due to Minkowski's inequality.
\end{remark}
We are now ready to state our first main convergence result.
\begin{cor}[Non-accelerated parameterization]
\label{cor:soda-nonacc-uniform}
Let $p\ge2$ and let $q=p/(p-1)$.
Let $x^\star\in\argmin_{x\in\operatorname{dom}h} f(x)$ and $D_\star:=h(x^\star)-\inf h$.
Consider \ref{eq:SODA}.
For every $k=0,\dots,n-1$, choose
\[
\alpha_k = \tfrac{1}{k+1},
\qquad
\bar\alpha_k = \lambda_k = \tfrac{1}{k+2},
\qquad
\bar\lambda_k \le \tfrac{\lambda_k}{6},
\qquad
\gamma_k = \eta(k+2).
\]
Suppose \Cref{ass:convex,ass:soda-stochastic,ass:smooth,ass:noise} hold with $q=p/(p-1)$.
Assume also that $h$ is $p$-uniformly convex with constant $\mu$ in the sense of \Cref{def:uniform-convexity}.
Choose
\[
\eta
\asymp_{p,\tau_q}
\min\left\{
\tfrac{\mu^{2/p}D_\star^{(p-2)/p}}{L},
\tfrac{\mu^{1/p}D_\star^{(p-1)/p}}
{\sigma_q\,n^{(p-1)/p}}
\right\}.
\]
Then, for every $n\ge1$,
\[
\mathbb E[f(x^{n-1})-f(x^\star)]
=
O_{p,\tau_q}\!\left(
\tfrac{L D_\star^{2/p}}{\mu^{2/p}n}
+
\tfrac{\sigma_q D_\star^{1/p}}{\mu^{1/p}n^{1/p}}
\right).
\]
In particular, for $h(x)=\tfrac1p\norm{x-z^0}_p^p$ and $R_p:=\norm{x^\star-z^0}_p$, taking $\mu=2^{2-p}$ gives
\[
\mathbb E[f(x^{n-1})-f(x^\star)]
=
O_{p,\tau_q}\!\left(
\tfrac{L_p R_p^2}{n}
+
\tfrac{\sigma_q R_p}{n^{1/p}}
\right).
\]
\end{cor}
For $p>2$, this follows from the more general \Cref{thm:soda-nonacc-holder-uniform}, which allows H\"older smoothness, by taking $\nu=1$ and $L_\nu=L$, while $p=2$ is the strongly convex SODA guarantee of \citet[Cor.~4.6]{pethick2026optimistic}.
\Cref{cor:soda-near-linf} in the appendix shows that the $\infty$-geometry can be approximated by $p=O(\log(d))$ with only a logarithmic loss.
For simplicity of presentation, in the remaining discussion involving $p\rightarrow \infty$, we refer to this result for $p=O(\log(d))$ and suppress the logarithmic dependence in the resulting rates and complexities.

\begin{figure}[t]
\centering
\begin{subfigure}[t]{0.49\linewidth}
\centering
\includegraphics[width=\linewidth]{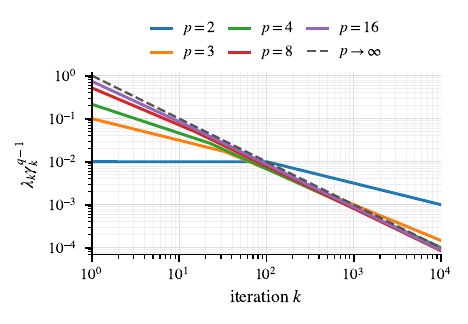}
\caption{Non-accelerated stepsize (\Cref{cor:soda-nonacc-uniform}).}
\label{fig:nonaccelerated-stepsize-scale}
\end{subfigure}\hfill
\begin{subfigure}[t]{0.49\linewidth}
\centering
\includegraphics[width=\linewidth]{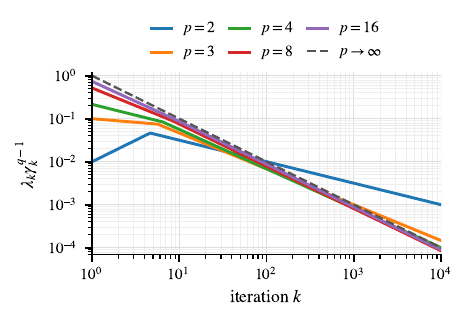}
\caption{Accelerated stepsize (\Cref{cor:soda-acc-uniform}).}
\label{fig:accelerated-stepsize-scale}
\end{subfigure}
\caption{Effective stepsize $\lambda_k\gamma_k^{q-1}$ for several
geometries, using $C_{\rm det}=10^{-2}$ and $C_{\rm stoch}=10^{-1}$.
Interestingly, we find a warmup-like phase for the accelerated parameterization, which is present in the Euclidean case but not for $p\rightarrow \infty$, consistent with practice (although the functional form differs).}
\label{fig:stepsize-scales}
\end{figure}

\paragraph{Effective stepsize}
Let us interpret the stepsize choices in \Cref{cor:soda-nonacc-uniform}.
For the finite-$p$ regularizer $h(x)=\tfrac1p\norm{x-z^0}_p^p$, the update for $z^{k+1}$ in \ref{eq:SODA} contains the stepsize $\gamma_k^{q-1}$.
To see how this stepsize changes with the geometry, consider the horizon-free version, where the horizon $n$ in $\eta$ is replaced by the current iteration $k$:\footnote{We leave explicit treatment of such anytime parameterization for future work.}
\[
\eta_k
\asymp
\min\left\{
C_{\rm det},
C_{\rm stoch}k^{-(p-1)/p}
\right\},
\quad \text{with}\quad
C_{\rm det}:=\tfrac{\mu^{2/p}D_\star^{(p-2)/p}}{L},
\quad
C_{\rm stoch}:=\tfrac{\mu^{1/p}D_\star^{(p-1)/p}}{\sigma_q}.
\]
Since $\gamma_k=\eta_k(k+2)$, this gives
\[
\gamma_k^{q-1}
\asymp
\min\left\{
C_{\rm det}^{1/(p-1)}k^{1/(p-1)},
C_{\rm stoch}^{1/(p-1)}k^{1/(p(p-1))}
\right\}.
\]
For the centered regularizer $h(x)=\tfrac1p\norm{x-z^0}_p^p$, the averaged iterate update in \ref{eq:SODA} reduces to

\equationbox{
\begin{equation*}
x^{k+1}
=(1-\lambda_k)x^k
+\lambda_k z^0
-\lambda_k\gamma_k^{q-1}
\nabla\!\left(\tfrac1q\norm{\cdot}_q^q\right)
(\bar m^{k+1}).
\end{equation*}
}
From this we see clearly that $\lambda_k$ acts as weight decay \citep{hanson1988comparing} while $\lambda_k\gamma_k^{q-1}$ acts as the effective stepsize of the update, which we plot in \Cref{fig:stepsize-scales}.
Specifically, at the endpoints we see the effect of $p$ on the effective stepsize
\[
p=2:\qquad
\lambda_k\gamma_k^{q-1}
\asymp
\min\left\{
C_{\rm det},
C_{\rm stoch}k^{-1/2}
\right\},
\qquad
p\rightarrow\infty:\qquad
\lambda_k\gamma_k^{q-1}
\asymp
k^{-1}.
\]

For a more aggressive hyperparameter choice we can accelerate the smooth term (depending on $p$).
\begin{cor}[Accelerated parameterization]
\label{cor:soda-acc-uniform}
Let $p\ge2$ and let $q=p/(p-1)$.
Let $x^\star\in\argmin_{x\in\operatorname{dom}h} f(x)$ and $D_\star:=h(x^\star)-\inf h$.
Consider \ref{eq:SODA}.
For every $k=0,\dots,n-1$, choose
\begin{align*}
\alpha_k = \tfrac{2}{k+2},
\qquad
\bar\alpha_k = \lambda_k = \tfrac{2}{k+3},
\qquad
\bar\lambda_k \le \tfrac{\lambda_k}{6},
\qquad
\gamma_k = \eta \tfrac{(k+2)(k+3)}{2}.
\end{align*}
Suppose \Cref{ass:convex,ass:soda-stochastic,ass:smooth,ass:noise} hold with $q=p/(p-1)$.
Assume also that $h$ is $p$-uniformly convex with constant $\mu$ in the sense of \Cref{def:uniform-convexity}.
Choose
\[
\eta
\asymp_{p,\tau_q}
\min\left\{
\tfrac{\mu^{2/p}D_\star^{(p-2)/p}}{L\,n^{(p-2)/p}},
\tfrac{\mu^{1/p}D_\star^{(p-1)/p}}
{\sigma_q\,n^{(2p-1)/p}}
\right\}.
\]
Then for every $n\ge1$, we have
\[
\mathbb E[f(x^{n-1})-f(x^\star)]
=
O_{p,\tau_q}\!\left(
\tfrac{L D_\star^{2/p}}{\mu^{2/p}n^{1+2/p}}
+
\tfrac{\sigma_q D_\star^{1/p}}{\mu^{1/p}n^{1/p}}
\right).
\]
In particular, for $h(x)=\tfrac1p\norm{x-z^0}_p^p$ and $R_p:=\norm{x^\star-z^0}_p$, taking $\mu=2^{2-p}$ gives
\[
\mathbb E[f(x^{n-1})-f(x^\star)]
=
O_{p,\tau_q}\!\left(
\tfrac{L_p R_p^2}{n^{1+2/p}}
+
\tfrac{\sigma_q R_p}{n^{1/p}}
\right).
\]
\end{cor}
For $p>2$, this follows from the more general \Cref{thm:soda-acc-holder-uniform}, which allows H\"older smoothness, by taking $\nu=1$ and $L_\nu=L$, while $p=2$ is the accelerated strongly convex SODA guarantee of \citet[Cor.~B.6]{pethick2026optimistic}.
For fixed finite $p\ge2$, the deterministic term is tight up to constants depending on $p$ (see \citet[Cor.~1]{guzman2018lower}).
For $p=\infty$, \citet[Cor.~1]{guzman2018lower} only gives tightness up to a logarithmic factor in the dimension, thus matching \Cref{cor:soda-near-linf}.

We can extract a remarkable number of insights from these bounds
as we will see next.

\subsection[Batch size under moment noise]{Batch size under moment noise}

The noise assumption in the convergence theorems is inherently tied to the algorithm norm, since choosing the primal geometry $\norm{\cdot}_p$ fixes the dual norm $\norm{\cdot}_q$, where $q=p/(p-1)$.
This results in an overly pessimistic oracle complexity when a stronger noise assumption is available.

For mini-batching, we therefore strengthen the oracle assumption by requiring an $r$-moment.
Larger $r$ means less heavy-tailed noise, and the extra moment is what allows the effective noise scale to decrease with the batch size.
We write
\[
p_r:=\tfrac{r}{r-1},
\qquad
\tfrac{1}{p_r}=\tfrac{r-1}{r}.
\]

\begin{assumption}[Mean-zero $r$-moment stochastic gradients]
\label{ass:mean-zero-r-moment}
For some $r\in(1,2]$, for every $x\in\mathcal X$,
\[
\mathbb E_\xi[\nabla f(x,\xi)]=\nabla f(x),
\qquad
\mathbb E_\xi\!\left[
\norm{\nabla f(x,\xi)-\nabla f(x)}_r^r
\right]\le \sigma_r^r .
\]
\end{assumption}
This implies the gradient-variation condition in \Cref{ass:noise} for $q=r$ (see \Cref{rem:gradvar}).

Here we keep the stochastic oracle structure explicit in order to track how batching changes the scale.
For a mini-batch average
\[
\textstyle \bar g_B(x):=\tfrac1B\sum_{i=1}^B \nabla f(x,\xi_i),
\]
the von Bahr--Esseen inequality \citep{vonBahrEsseen1965,Pinelis1986} gives
\[
\mathbb E[\norm{\bar g_B(x)-\nabla f(x)}_r^r\mid x]^{1/r}
\lesssim_{r}
\sigma_rB^{-1/p_r}.
\]
Let $d$ denote the ambient dimension, or the effective rank in the Schatten case.
If $q\le r$, then
\begin{equation}
\label{eq:q-r-norm-conversion}
\norm{u}_q\le d^{1/q-1/r}\norm{u}_r.
\end{equation}
Thus the only role of batching in the rate below is to replace $\sigma_q$ by $d^{1/q-1/r}\sigma_rB^{-1/p_r}$, while keeping $\tau_q=1$.
We are now ready to state our complexity result.

\begin{cor}[Oracle complexity with less heavy-tailed noise]
\label{prop:minibatch-finite-p}
Let $p\ge2$, set $q=p/(p-1)$, and suppose $p\ge p_r$.
Suppose \Cref{ass:mean-zero-r-moment} holds.
Run \ref{eq:SODA} with gradient feedback $g^k=\bar g_B(y^k)$.
Then \Cref{cor:soda-nonacc-uniform} applies with its noise parameter $\sigma_q$ replaced by $d^{1/q-1/r}\sigma_rB^{-1/p_r}$, giving
\[
\mathbb E[f(x^{n-1})-f(x^\star)]
\lesssim
\tfrac{A_p}{n}
+
\tfrac{S_{r,p}}{B^{1/p_r}n^{1/p}},
\]
where $R_p:=\norm{x^\star-z^0}_p$, $A_p=L_pR_p^2$, and $S_{r,p}:=\sigma_r d^{1/q-1/r}R_p$.
Then accuracy $\epsilon$ is reached by taking
\[
n\asymp
\tfrac{A_p}{\epsilon},
\qquad
B\asymp
\left(\tfrac{S_{r,p}}{\epsilon}\right)^{p_r}
\left(\tfrac{\epsilon}{A_p}\right)^{p_r/p},
\]
and hence with total stochastic gradient budget
\[
N=nB
=
O\!\left(
\tfrac{A_p}{\epsilon}
+
\left(\tfrac{S_{r,p}}{\epsilon}\right)^{p_r}
\left(\tfrac{A_p}{\epsilon}\right)^{1-p_r/p}
\right).
\]
\end{cor}
There are a couple of insights we can extract from this complexity result.
\paragraph{Batch size rule for a fixed geometry}
Suppose the total stochastic gradient budget $N=nB$ is fixed instead of the accuracy $\epsilon$.
Substituting $B=N/n$ in the rate of \Cref{prop:minibatch-finite-p} gives
\[
\epsilon_p(n,N)
\lesssim
\tfrac{A_p}{n}
+
\tfrac{S_{r,p}}{N^{1/p_r}}\,n^{1/p_r-1/p},
\qquad p\ge p_r .
\]
At the smallest admissible geometry $p=p_r$, the stochastic term is independent of the iteration count after the substitution $B=N/n$.
Thus any batch size choice with
\[
\tfrac{A_{p_r}}{n}
\lesssim
\tfrac{S_{r,p_r}}{N^{1/p_r}}
\qquad\Longleftrightarrow\qquad
1\le B\lesssim \tfrac{S_{r,p_r}}{A_{p_r}}N^{1-1/p_r},
\]
attains the same oracle complexity.
This is otherwise not the case.

For $p>p_r$, balancing the smooth and stochastic terms gives the following scaling rule:
\equationbox{
\begin{equation}\label{eq:bs-scaling-rule} 
n_p
\asymp
\left(\tfrac{A_p}{S_{r,p}}\right)^{1/(1+1/p_r-1/p)}
N^{\frac{1/p_r}{1+1/p_r-1/p}},
\qquad
B_p=\tfrac{N}{n_p}
\asymp
\left(\tfrac{S_{r,p}}{A_p}\right)^{1/(1+1/p_r-1/p)}
N^{\frac{1-1/p}{1+1/p_r-1/p}}.
\end{equation}
}

\begin{table}[t]
\centering
\begingroup
\makeatletter
\let\hyper@anchorstart\@gobble
\let\hyper@anchorend\relax
\caption{Batch size scaling rules for a given geometry under bounded variance. Here
$N=nB$ is the total number of stochastic gradient oracle calls, which corresponds
to the total token budget in LLMs. 
The batch size increases with $p$ and recovers the BST rule for $p\rightarrow\infty$ \citep{islamov2026batchsize}.}
\label{tab:bounded-var-batchsize}
\makeatother
\endgroup
\begin{minipage}[c]{0.68\linewidth}
\centering
\scriptsize
\setlength{\tabcolsep}{2pt}
\begin{tabular}{@{}p{0.16\linewidth}p{0.21\linewidth}p{0.28\linewidth}p{0.29\linewidth}@{}}
\toprule
\textbf{Geometry} & \textbf{Iterations} & \textbf{Batch size} & \textbf{Rate} \\
\midrule
$p=2$ &
$n_2\asymp N^{1/2}$ &
$B_2\asymp N^{1/2}$ &
$\epsilon_N\asymp N^{-1/2}$ \\
$2<p<\infty$ &
$n_p\asymp N^{p/(3p-2)}$ &
$B_p\asymp N^{2(p-1)/(3p-2)}$ &
$\epsilon_N\asymp N^{-p/(3p-2)}$ \\
$p\rightarrow \infty$ &
$n_\infty\asymp N^{1/3}$ &
$B_\infty\asymp N^{2/3}$ &
$\epsilon_N\asymp N^{-1/3}$ \\
\bottomrule
\end{tabular}
\end{minipage}\hfill
\begin{minipage}[c]{0.32\linewidth}
\centering
\includegraphics[width=\linewidth]{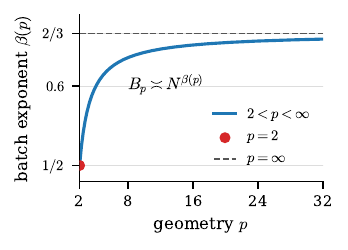}
\end{minipage}
\end{table}

\paragraph{Bounded variance}
Let us get a feeling for this scaling by developing the special case of bounded variance, $r=2$ and hence $p_r=2$.
The corresponding batch size scaling rules based on \eqref{eq:bs-scaling-rule} are summarized in \Cref{tab:bounded-var-batchsize}.
Interestingly, for $p\rightarrow\infty$ we recover the BST batch size scaling rule of \citet{islamov2026batchsize}, which identically uses $B_\infty \asymp N^{2/3}$.

Specifically, for a fixed oracle budget $N$, the two endpoint stochastic error rates are
\[
p=2:\quad \epsilon_N=O(N^{-1/2}),
\qquad
p\rightarrow\infty:\quad \epsilon_N=O(d^{1/3}N^{-1/3}).
\]
From the above it might appear as if the Euclidean geometry is always better than the $\infty$-geometry, both in terms of the dependence on the oracle budget $N$ and the dimension $d$.
However, this comparison does not take into account that the smoothness and radius terms themselves depend on the geometry.
In the next section we show that, when the objective is smooth in the $\infty$-geometry, different regimes emerge depending on $d$ and $N$.

\begin{remark}
We have focused on optimizing the number of oracle calls $N$ which corresponds to optimizing the total computation budget.
One could equally well consider fixing the number of iterations $n$ and hence the wall-clock time (given infinite compute).
\end{remark}

\section{Regimes}

\subsection[Choosing p under infinity smoothness]{Choosing $p$ under $\infty$-smoothness}

We now ask what the choice of the geometry for the algorithm should be in deep learning based on the derived complexities.
Let us have some faith and assume that the neural network is smooth in the Schatten-$\infty$, as e.g., argued in the modular norm work \citep{large2024scalable,bernstein2024modular}.
Maybe surprisingly, we will see that the optimal choice for the algorithm is not necessarily the matching Schatten-$\infty$ geometry.

Assume that $f$ is $L_\infty$-smooth with respect to $\norm{\cdot}_\infty$, and that the mini-batch oracle satisfies \Cref{ass:mean-zero-r-moment} with a fixed $r\in(1,2]$.
The oracle fixes the threshold $p_r=r/(r-1)$ for which the $p$-geometry method is admissible when $p\ge p_r$.

The $\infty$-smoothness assumption is inherited by every finite-$p$ analysis geometry, since
\[
\norm{\nabla f(x)-\nabla f(y)}_{p/(p-1)}
\le
\norm{\nabla f(x)-\nabla f(y)}_1
\le
L_\infty\norm{x-y}_\infty
\le
L_\infty\norm{x-y}_p.
\]
Also $R_p:=\norm{x^\star-z^0}_p\le d^{1/p}R_\infty$, which shows the dimension price of using a smaller $p$.

\subsubsection{Deterministic setting and acceleration}

A tradeoff can already be seen in the deterministic setting.
When the noise is negligible, we take $B=1$ and identify the oracle budget with the number of iterations, $N=n$.
The accelerated rate (\Cref{cor:soda-acc-uniform}) then gives
\equationbox{
\[
\epsilon^{\rm det}_p(N)
\lesssim
\tfrac{L_\infty R_\infty^2d^{2/p}}{N^{1+2/p}}
=\tfrac{L_\infty R_\infty^2}{N}\left(\tfrac{d}{N}\right)^{2/p}.
\]
}
Clearly, when $N\gtrsim d$, picking $p$ small improves the rate through $(\tfrac{d}{N})^{2/p}$.
On the other hand, when $N\lesssim d$, the same factor penalizes small $p$ and large $p \rightarrow \infty$ is preferred.
The deterministic regimes are summarized in \Cref{tab:linf-p-choice}.

\subsubsection{Heavy-tailed noise and batching}

In the noisy setting, the accelerated complexity result in \Cref{cor:acc-minibatch-finite-p} gives the same kind of $d$-$N$ separation.
Under $\infty$-smoothness, the only additional estimates are
\[
A_p=L_pR_p^2\lesssim L_\infty R_\infty^2d^{2/p},
\qquad
S_{r,p}=\sigma_rd^{1/q-1/r}R_p\lesssim \sigma_rR_\infty d^{1/p_r}.
\]
Write $A:=L_\infty R_\infty^2$ and $S:=\sigma_rR_\infty$.
Thus, after setting $B=N/n$, fixing the budget $N$ in \Cref{cor:acc-minibatch-finite-p} gives a rate of
\[
\epsilon^{\rm stoch}_p(n,N)
\lesssim
\tfrac{A d^{2/p}}{n^{1+2/p}}
+
\tfrac{S d^{1/p_r}}{N^{1/p_r}}\,
n^{1/p_r-1/p},
\qquad
p\ge p_r .
\]

At the smallest admissible geometry, taking $p=p_r$ and $B=1$ gives
\[
\epsilon^{\rm stoch}_{p_r}(N)
\lesssim
A\,d^{-1}\left(\tfrac{d}{N}\right)^{1+2/p_r}
+
S\left(\tfrac{d}{N}\right)^{1/p_r}.
\]
At the other endpoint $p\rightarrow \infty$, fixing the budget $N$ and optimizing over $n$ gives
\[
\epsilon^{\rm stoch}_\infty(N)
\lesssim
A\,d^{-1}\left(\tfrac{d}{N}\right)
+
A^{1/(1+p_r)}S^{p_r/(1+p_r)}
\left(\tfrac{d}{N}\right)^{1/(1+p_r)}.
\]
Thus the accelerated term is clearly better for $p=p_r$ once $N\gtrsim d$.
The stochastic terms give the same comparison once
\[
\tfrac{d}{N}
\lesssim
\left(\tfrac{A}{S}\right)^{p_r},
\qquad\text{equivalently}\qquad
N\gtrsim d\left(\tfrac{S}{A}\right)^{p_r}.
\]
Combining the deterministic and stochastic comparisons, the smallest admissible geometry $p=p_r$ is favored when
\equationbox{
\[
N\gtrsim
d\max\left\{1,\left(\tfrac{S}{A}\right)^{p_r}\right\}.
\]
}
This is again the low-dimensional regime $N\gtrsim d$, as summarized in \Cref{tab:linf-p-choice}.

\paragraph{Limitations}
To compare the different methods we have used worst-case norm conversions both for the radius, from $R_{p_r}$ to $R_\infty$, and for the noise, from the algorithmic dual norm $\norm{\cdot}_q$ to the assumed $r$-moment norm \eqref{eq:q-r-norm-conversion}, both of which introduce a dimensionality dependency.
Sharper problem-dependent relations can improve these constants and move the boundary.

\begin{table}[t]
\centering
\footnotesize
\begingroup
\makeatletter
\let\hyper@anchorstart\@gobble
\let\hyper@anchorend\relax
\caption{Optimal choice of geometry under smoothness in $p=\infty$.
Matching the problem geometry is favorable in the high-dimensional regime ($N\lesssim d$). 
However, in the low-dimensional regime ($N\gtrsim d$) choosing a smaller $p$ is favorable both deterministically and in the presence of heavy-tailed noise.
}
\label{tab:linf-p-choice}
\makeatother
\endgroup
\begin{tabular}{@{}p{0.29\linewidth}p{0.40\linewidth}p{0.23\linewidth}@{}}
\toprule
\textbf{Setting} & \textbf{Regime} & \textbf{Preferred geometry} \\
\midrule
Deterministic ($\sigma_q=0$) &
$N\gtrsim d$ &
$p=2$ \\
Deterministic ($\sigma_q=0$) &
$N\lesssim d$ &
$p\rightarrow \infty$ \\
Heavy-tailed noise ($\sigma_r>0$) &
$N\gtrsim d \max\left\{1,\left(\tfrac{S}{A}\right)^{p_r}\right\}$ &
$p=p_r$ \\
Heavy-tailed noise ($\sigma_r>0$) &
$N\lesssim d \max\left\{1,\left(\tfrac{S}{A}\right)^{p_r}\right\}$ &
$p\rightarrow \infty$ \\
\bottomrule
\end{tabular}
\end{table}

\subsubsection{Anytime switching under heavy-tailed noise}

The careful reader might have noticed that the above regime discussion requires knowing the budget $N$ in order to pick the optimal geometry $p$.
What if we did not know $N$?
The regime bound suggests switching from $p$ close to $\infty$ to $p=p_r$ once
\[
N_{\rm sw}
\asymp
d\max\left\{1,\left(\tfrac{S}{A}\right)^{p_r}\right\}.
\]
When the final oracle budget is not known in advance, this gives the rule
\[
p(N)
=
\begin{cases}
p_d, & N<N_{\rm sw},\\
p_r, & N\ge N_{\rm sw},
\end{cases}
\qquad
p_d=\max\{3,\lceil\log(\mathrm e d)\rceil\}.
\]
This is the optimization analogue of the adaptive-regularization principle in anytime online learning on $\ell_p$ balls \citep{johnson2025adaptive}.
Concretely, use the following restart procedure:
\begin{enumerate}[leftmargin=*,itemsep=0pt,topsep=2pt]
\item Run \ref{eq:SODA} with $p=p_d$ using the parameterization from \Cref{cor:soda-near-linf} until $N_{\rm sw}$ oracle calls, choosing the batch size by optimizing the complexity bound at the switching budget.
\item Restart with $p=p_r$ using the parameterization from \Cref{cor:acc-minibatch-finite-p}, keeping the same batch size.
\end{enumerate}

Remarkably, Soft-Muon \citep{nilin2024contramuon} seems to have arrived at such a geometry-schedule empirically by softening the $\infty$-geometry later in training.\footnote{See the \href{https://github.com/KellerJordan/modded-nanogpt/pull/291}{speedrun submission}.} In contrast, the proposal above is a \emph{discrete} switching strategy motivated by the regime discussion.
We leave further practical exploration of this idea to future work.

\subsection{Implications for scaling laws}

We now ask what can be extracted from the regime discussion above in the context of deep learning.
A typical scaling-law prescription, such as Chinchilla \citep{hoffmann2022training}, is that the model size should scale proportionally to the number of training tokens/stochastic gradient oracle calls,
\[
\texttt{model\_size} \propto N.
\]
For a matrix parameter in a depth-$D$ network, however, the model size is
\[
\texttt{model\_size}\propto \texttt{fan\_in}*\texttt{fan\_out}*D .
\]
If width and depth are scaled together, then the rank of each matrix grows only like
\[
\texttt{rank}\ \propto\ (\texttt{model\_size})^{1/3}\ \propto\ N^{1/3}.
\]
For Schatten-$p$ geometries, the rank is the effective dimension $d$ in the regime discussion.
The scaling law puts us well within the low-dimensional regime $N\gtrsim d$, where it is optimal to pick the smallest admissible geometry, Schatten-$p_r$, based on the heavy-tailed noise assumption (see \Cref{tab:linf-p-choice}).

The Chinchilla scaling therefore lies in the low-dimensional regime, where smaller $p$ is favorable.
This may explain why the improvement for Schatten-$\infty$ based methods such as Muon over baselines has been observed to reduce with scale \citep{wen2025fantastic,semenov2025benchmarking} and the recent interest in finite Schatten-$p$ norm optimizers \citep{nilin2024contramuon,pang2026htmuon,shumaylov2026muon}.
As promising future work we leave open the question of how to scale more appropriately for Schatten-$\infty$ based optimizers.

%% file: Sections/RelatedWork.tex
The regime discussion is inspired by \citet{johnson2025adaptive}, who study FTRL in online learning and identify when the optimal geometry separates from the Euclidean choice $p=2$.
In comparison, we study an optimistic accelerated algorithm in the offline stochastic setting under heavy-tailed noise with moment $r$, and ask when the smallest admissible geometry $p_r$ separates from the near-$\infty$ geometry.
\looseness=-1

The deterministic accelerated rate is related to non-Euclidean acceleration and optimistic online-to-batch conversion.
\citet{daspremont2018optimal} study accelerated methods under $p$-uniform convexity, showing the same $p$-dependent deterministic rate of $L_p R_p^2/n^{1+2/p}$ on $\ell_p$ balls, matching the lower bounds of \citet{guzman2018lower}.
Another route to acceleration was pioneered in the online-learning literature through online-to-offline and adaptive universality results \citep{Lev17,LYC18}, the constrained UnixGrad framework \citep{KLBC19}, and the optimistic online-to-batch constructions of \citet{cutkosky2019anytime,joulani2020simpler}.
This perspective was later used in the schedule-free framework of \citet{defazio2024road}.
Our analysis follows this latter route, but uses Optimistic Dual Averaging with a $p$-uniformly convex reference function, which changes the Euclidean accelerated term $L R^2/n^2$ into $L_p R_p^2/n^{1+2/p}$ and the noise term to be based on a heavy-tailed assumption.
\looseness=-1

Uniform convexity was used to analyze Online Mirror Descent in the context of online learning in \citet{sridharan2010convex,srebro2011universality}.
In stochastic convex optimization with infinite variance, \citet{vural2022mirror} prove optimal rates for stochastic mirror descent with uniformly convex reference functions under heavy-tailed stochastic gradients with bounded $r$-th moment, $r\in(1,2]$.
In Euclidean geometry, much of the theory handles such noise by modifying the gradients, e.g, \citet{cutkosky2021highprobability} combine clipping, momentum, and normalized gradient descent and \citet{gorbunov2020stochastic} use accelerated gradient clipping.
More recent work studies when clipping can be avoided \citep{hubler2025gradient,liu2025nonconvex}.

Muon can be viewed and analyzed as a conditional gradient method \citep{pethick2025trainingdeeplearningmodels,StochasticConditionalGradient2020mokhtari}.
Specifically with weight decay, Muon reduces to stochastic Frank-Wolfe for which \citet{UnderstandingGradientOrthogonalization2025kovaleva} additionally provided analysis under star-convexity.
This constrained endpoint is recovered by our template in the limit $p\to\infty$ with primal extrapolation ($\bar\lambda_k=0$) and no optimism ($\bar\alpha_k=0$).
Muon can also be seen as an instance of the Lion-$\mathcal K$ framework, which allows capturing what we refer to as optimism \citep{MuonOptimizesSpectral2025chena,LionsMuonsOptimization2025sfyraki}.
Our analysis differs by focusing on finite $p$, where we show that optimism helps exploit smoothness and accelerate convergence.
For Schatten-$\infty$ geometries we approximate with $p=O(\log d)$, matching practice where $UV^\top$ is only computed approximately.

%% file: Sections/Conclusion.tex
We have learned that the choice of $p$-norm to use in the algorithm depends on the training regime.
For high-dimensional problems ($N\lesssim d$) the optimizer should match the smoothness assumption, while for low-dimensional problems ($N\gtrsim d$) the optimizer should instead match the noise assumption in the stochastic case and use the Euclidean geometry in the deterministic case to accelerate.

What is striking about our results is how closely they match empirical observations.
Our batch size analysis provides an explanation for why Schatten-$\infty$ based optimizers like Muon typically require larger batch sizes and gives a batch size scaling rule for any Schatten-$p$ norm.
We have also seen an explanation for why stepsize warmup is not needed for the Schatten-$\infty$ norm, quantified when a smaller Schatten-$p$ norm is beneficial and provided theoretical backing for existing empirical schedules of $p$.

%% file: Sections/Acknowledgments.tex
The author thanks the Swiss chômage for supporting the work.

%% file: Sections/AnalysisV2Appendix.tex
\subsection{Preliminaries}

We use the following notation.
For a norm $\norm{\cdot}$ on the primal space, the dual norm is
\[
\norm{g}_*:=\sup_{\norm{x}\le1}\braket{g,x}.
\]
For a proper function $h:\mathcal X\to\R\cup\{\infty\}$, its domain is
\[
\operatorname{dom}h:=\{x\in\mathcal X:h(x)<\infty\}.
\]
Constraints are supported by taking the reference function to be $h(x)=\psi(x)+\iota_{\mathcal C}(x)$, where $\iota_{\mathcal C}$ is the indicator of the constraint set, equal to $0$ on $\mathcal C$ and $+\infty$ otherwise.
Then $\operatorname{dom}h=\mathcal C\cap\operatorname{dom}\psi$, and the dual map can be written as
\[
\partial h^*(-\gamma m)
=
\argmin_{x\in\mathcal X}\{\gamma\braket{m,x}+h(x)\}
=
\argmin_{x\in\mathcal C}\{\gamma\braket{m,x}+\psi(x)\}.
\]

For a matrix $X$, let $\sigma_i(X)$ denote its singular values and define the Schatten-$p$ norm by
\[
\norm{X}_{S_p}:=\left(\textstyle\sum_i \sigma_i(X)^p\right)^{1/p},
\qquad
\norm{X}_{S_\infty}:=\max_i\sigma_i(X).
\]
The dual of $\norm{\cdot}_{S_p}$ is $\norm{\cdot}_{S_q}$ when $q=p/(p-1)$.

\paragraph{Additional $p$-uniformly convex norms.}
\label{app:additional-p-uniform-norms}
The same framework is not restricted to vector $\ell_p$ and Schatten-$p$ norms (see \citet{ball1994sharp} for Schatten-$p$ norms).
For a matrix $X$ with rows $X_i$, define the mixed row norm
\[
\norm{X}_{2,p}:=\left(\textstyle\sum_i \norm{X_i}_2^p\right)^{1/p},
\qquad
\norm{X}_{2,\infty}:=\max_i\norm{X_i}_2 .
\]
The endpoint $\norm{\cdot}_{2,\infty}$ is the operator norm $\norm{\cdot}_{2\to\infty}$.
For $p>2$, $h(X)=\tfrac1p\norm{X-Z^0}_{2,p}^p$ is $p$-uniformly convex with respect to $\norm{\cdot}_{2,p}$ with admissible constant $\mu=2^{2-p}$, since it is an $\ell_p$ sum of row norms.
Thus the results below also apply to the mixed-norm geometry after replacing $\norm{\cdot}_p$ and its dual by $\norm{\cdot}_{2,p}$ and $\norm{\cdot}_{2,q}$.

For positive quantities $A$ and $B$, $A\lesssim_\Theta B$ means $A\le C(\Theta)B$ for a constant depending only on the parameters $\Theta$, and $A\gtrsim_\Theta B$ means $B\lesssim_\Theta A$.
We write $A\asymp_\Theta B$ when both $A\lesssim_\Theta B$ and $A\gtrsim_\Theta B$ hold, and $A=O_\Theta(B)$ with the same meaning as $A\lesssim_\Theta B$.

\subsection{Technical lemmas}

Our proof builds on \citet{defazio2024road}.
However, rather than combining primal extrapolation with an adaptive version of Optimistic Mirror Descent, we use Optimistic Dual Averaging \citep{rakhlin2013online} as the underlying no-regret algorithm.
The ODA analysis is based on a non-adaptive version of \citet{rakhlin2013online}, which is classical (see, e.g., \citet[Sec.~7.12]{Orabona2019}).
We furthermore generalize the result to uniform convex reference functions, H\"older smoothness and a heavy-tailed type noise assumption, which are essential for our regime discussion.

\begin{assumption}[$(L_\nu,\nu)$-H\"older smooth]
\label{ass:holder-smooth}
For some $\nu\in(0,1]$, the function $f$ is differentiable and satisfies
\[
\norm{\nabla f(u)-\nabla f(v)}_*
\le
L_\nu \norm{u-v}^{\nu}
\qquad
\text{for all }u,v\in\mathcal X .
\]
\end{assumption}
When specializing \Cref{ass:holder-smooth} to $\norm{\cdot}_p$, we write $L_{\nu,p}$ for the corresponding H\"older smoothness constant and use the shorthand $L_p:=L_{1,p}$ in the smooth case.

\begin{lemma}[ODA regret with a uniformly convex reference]
\label{lem:oda-regret}
Consider \ref{eq:SODA}, with the gradient term $\nabla f(y^{k},\xi_{k})$ replaced by $g^{k}$.
Choose positive weights $a_0,\dots,a_{n-1}$ and write
\[
A_k := \textstyle \sum_{i=0}^{k} a_i,
\qquad A_{-1}:=0.
\]
We use the convention $g^{-1}=0$ and set $z^0\in\partial h^*(0)$.
For $k=0,\dots,n-2$, choose
\[
\alpha_k = \tfrac{a_k}{A_k},
\quad
\bar\alpha_k = \tfrac{a_{k+1}}{A_{k+1}},
\quad
\gamma_k=\eta \tfrac{A_k}{1-\bar\alpha_k}
\]
for some $\eta>0$.
Let $p\ge2$, let $q=p/(p-1)$, and let $h$ be $p$-uniformly convex with constant $\mu$ in the sense of \Cref{def:uniform-convexity}.
Then, for every $x \in \mathcal X$,
\begin{equation}
\label{eq:oda-regret}
\textstyle \sum_{k=0}^{n-1} a_k \braket{g^k, z^k-x}
\le
\tfrac{h(x)-\inf_{u \in \mathcal X}h(u)}{\eta}
+
\tfrac{\eta^{q-1}}{q\mu^{q-1}} \textstyle \sum_{k=0}^{n-1} a_k^q \norm{g^k-g^{k-1}}_*^q.
\end{equation}
\end{lemma}

\begin{proof}
By the choice of $\gamma_k$ and $\bar\alpha_k$ we have, for $k=0,\dots,n-2$,
\[
\gamma_k(1-\bar\alpha_k)=\eta A_k,
\quad\text{and}\quad
\gamma_k\bar\alpha_k=\eta a_{k+1}.
\]
Combined with the choice of $\alpha_k$, the dual update can be written as
\[
z^k\in
\partial h^*\!\left(
-\eta \textstyle \sum_{i=0}^{k-1}a_i g^i
-\eta a_k g^{k-1}
\right),
\qquad k=0,\dots,n-1,
\]
where the sum is empty when $k=0$ and the case $k=0$ is exactly $z^0\in\partial h^*(0)$ because $g^{-1}=0$.
Set
\[
\theta^k := -\eta \textstyle \sum_{i=0}^{k} a_i g^i,
\qquad
\theta^{-1}:=0,
\qquad
\hat\theta^k := \theta^{k-1} - \eta a_k g^{k-1}.
\]
Thus $z^k \in \partial h^*(\hat\theta^k)$ and
\[
\theta^k = \hat\theta^k - \eta a_k(g^k-g^{k-1}).
\]
Since $h$ is $p$-uniformly convex, $h^*$ is $q$-uniformly smooth:
\[
h^*(\theta+\Delta)
\le
h^*(\theta)+\braket{\nabla h^*(\theta),\Delta}
+\tfrac{1}{q\mu^{q-1}}\norm{\Delta}_*^q .
\]
Hence
\[
h^*(\theta^k)
\le
h^*(\hat\theta^k)
-
\eta a_k \braket{g^k-g^{k-1}, z^k}
+
\tfrac{\eta^q a_k^q}{q\mu^{q-1}}\norm{g^k-g^{k-1}}_*^q.
\]
By Fenchel--Young,
\[
\braket{\hat\theta^k, z^k} = h(z^k)+h^*(\hat\theta^k),
\qquad
\braket{\theta^k, x} \le h(x)+h^*(\theta^k).
\]
Therefore
\[
\begin{aligned}
\eta a_k \braket{g^k, z^k-x}
&=
\eta a_k \braket{g^k-g^{k-1}, z^k}
+
\eta a_k \braket{g^{k-1}, z^k}
-
\eta a_k \braket{g^k, x}
\\
&=
\eta a_k \braket{g^k-g^{k-1}, z^k}
+
\braket{\theta^{k-1}-\hat\theta^k, z^k}
+
\braket{\theta^k-\theta^{k-1}, x}
\\
&=
\eta a_k \braket{g^k-g^{k-1}, z^k}
+
\braket{\theta^{k-1}, z^k}
-
\braket{\hat\theta^k, z^k}
+
\braket{\theta^k-\theta^{k-1}, x}
\\
&=
\eta a_k \braket{g^k-g^{k-1}, z^k}
+
\braket{\theta^{k-1}, z^k}
-
h(z^k)
-
h^*(\hat\theta^k)
+
\braket{\theta^k-\theta^{k-1}, x}
\\
&\le
h^*(\theta^{k-1}) - h^*(\theta^k)
+
\tfrac{\eta^q a_k^q}{q\mu^{q-1}}\norm{g^k-g^{k-1}}_*^q
+
\braket{\theta^k-\theta^{k-1}, x}.
\end{aligned}
\]
Summing over $k=0,\dots,n-1$ yields
\[
\eta \textstyle \sum_{k=0}^{n-1} a_k \braket{g^k, z^k-x}
\le
h^*(0)-h^*(\theta^{n-1})
+
\braket{\theta^{n-1}, x}
+
\tfrac{\eta^q}{q\mu^{q-1}} \textstyle \sum_{k=0}^{n-1} a_k^q \norm{g^k-g^{k-1}}_*^q.
\]
Finally, $h^*(0)=-\inf_{u\in\mathcal X}h(u)$ and $\braket{\theta^{n-1}, x} - h^*(\theta^{n-1}) \le h(x)$, so dividing by $\eta$ proves the claim.
\end{proof}

\begin{lemma}[SODA online-to-batch H\"older refinement]
\label{lem:soda-otb-holder}
Consider \ref{eq:SODA} with positive weights $a_0,\dots,a_{n-1}$ and
\[
\lambda_{k-1} = \tfrac{a_k}{A_k},
\qquad
\bar\lambda_{k-1} \le c_\nu \tfrac{a_k}{A_k},
\qquad
A_k := \textstyle \sum_{i=0}^{k} a_i,
\qquad
k=1,\dots,n-1.
\]
for some $c_\nu>0$ to be defined.
Suppose \Cref{ass:convex,ass:soda-stochastic,ass:holder-smooth} hold and let
\[
r_\nu := \tfrac{1+\nu}{\nu}.
\]
Then there exists a constant $c_\nu>0$, depending only on $\nu$, with $c_1=1/6$, such that, for every $x \in \mathcal X$, \ref{eq:SODA} satisfies
\begin{equation}
\label{eq:soda-otb-holder}
\begin{aligned}
A_{n-1} \, \mathbb E[f(x^{n-1})-f(x)]
&\le
\mathbb E\left[\textstyle \sum_{k=0}^{n-1} a_k \braket{g^k, z^k-x}\right] \\
&\qquad -
\tfrac{c_\nu}{L_\nu^{1/\nu}} \textstyle \sum_{k=0}^{n-1} A_{k-1}
\mathbb E\!\left[
\norm{\nabla f(y^k)-\nabla f(y^{k-1})}_*^{r_\nu}
\right].
\end{aligned}
\end{equation}
\end{lemma}

\begin{proof}
The choice $\lambda_{k-1}=a_k/A_k$ implies
\[
x^k = \tfrac{1}{A_k} \textstyle \sum_{i=0}^k a_i z^i,
\qquad
k=0,\dots,n-1.
\]
For differentiable $f$, write
\[
D_f(u,v) := f(u)-f(v)-\braket{\nabla f(v), u-v}.
\]
For $k=0,\dots,n-1$, following \citet[Thm. 5]{defazio2024road}, a careful expansion gives
	\[
	\begin{aligned}
	A_k\bigl(f(x^k)-f(x)\bigr)
- A_{k-1}\bigl(f(x^{k-1})-f(x)\bigr)
&=
a_k\braket{\nabla f(y^k), z^k-x}
\\
&\qquad
-
\tfrac{a_k}{\bar\lambda_{k-1}} D_f(y^k,x^k)
-
\tfrac{a_k(1-\bar\lambda_{k-1})}{\bar\lambda_{k-1}} D_f(x^k,y^k)
\\
&\qquad
-
A_{k-1} D_f(x^{k-1},x^k)
-
a_k D_f(x,y^k),
\end{aligned}
	\]
	with the convention that the terms involving $A_{-1}$ vanish, and that the two terms involving $\bar\lambda_{k-1}^{-1}$ vanish when $k=0$ or $\bar\lambda_{k-1}=0$, since then $y^k=x^k$.

    Since $z^k$ is computed from $\hat\theta^k$, it depends only on $g^0,\ldots,g^{k-1}$.
    Hence $z^k$ and $y^k$ are $\mathcal F_{k-1}$-measurable, and
\[
\mathbb E\!\left[
\braket{g^k,z^k-x}\mid\mathcal F_{k-1}
\right]
=
\braket{\nabla f(y^k),z^k-x}.
\]
Summing over $k=0,\dots,n-1$ and taking expectations gives
\[
\begin{aligned}
A_{n-1} \mathbb E[f(x^{n-1})-f(x)]
&\le
\mathbb E\left[\textstyle \sum_{k=0}^{n-1} a_k \braket{g^k, z^k-x}\right]
\\
&\qquad
- \textstyle \sum_{k=0}^{n-1}
\Bigl(
\tfrac{a_k}{\bar\lambda_{k-1}} \mathbb E D_f(y^k,x^k)
+
\tfrac{a_k(1-\bar\lambda_{k-1})}{\bar\lambda_{k-1}} \mathbb E D_f(x^k,y^k)
\Bigr)
\\
&\qquad
- \textstyle \sum_{k=0}^{n-1}
\Bigl(
A_{k-1}\mathbb E D_f(x^{k-1},x^k)
+ a_k\mathbb E D_f(x,y^k)
\Bigr).
\end{aligned}
\numberthis \label{eq:expectation}
\]
    \Cref{ass:convex,ass:holder-smooth} imply the generalized co-coercivity inequality
\[
D_f(u,v)
\ge
\tfrac{\nu}{1+\nu}L_\nu^{-1/\nu}
\norm{\nabla f(u)-\nabla f(v)}_*^{(1+\nu)/\nu}.
\]
    For $k\ge1$, let
	\[
	\begin{aligned}
	U_k&:=\norm{\nabla f(y^k)-\nabla f(x^k)}_*^{r_\nu},\\
V_k&:=\norm{\nabla f(x^k)-\nabla f(x^{k-1})}_*^{r_\nu},\\
W_k&:=\norm{\nabla f(y^k)-\nabla f(y^{k-1})}_*^{r_\nu}.
\end{aligned}
\]
The relation $x^0=y^0$ gives $U_0=0$.
Also,
\begin{equation}
\label{eq:holder-w-triangle}
	\begin{aligned}
	W_k
	&=
	\norm{
	\nabla f(y^k)-\nabla f(x^k)
	+\nabla f(x^k)-\nabla f(x^{k-1})
	+\nabla f(x^{k-1})-\nabla f(y^{k-1})
	}_*^{r_\nu}
	\\
	&\le
	K_\nu(U_k+V_k+U_{k-1}),
	\end{aligned}
\end{equation}
where $K_\nu=3^{r_\nu-1}$.
Moreover, $\bar\lambda_{k-1}\le c_\nu a_k/A_k$ implies $a_k/\bar\lambda_{k-1}\ge A_k/c_\nu$.
Let
	\[
	b_\nu:=\tfrac{\nu}{1+\nu},
	\qquad
	\beta_\nu:=\tfrac{b_\nu}{L_\nu^{1/\nu}}.
	\]
	Choose $c_\nu>0$ sufficiently small, depending only on $\nu$, and set $C_\nu:=c_\nu K_\nu$.
	Then, for $k\ge1$,
	\[
	\begin{aligned}
	&-\tfrac{a_k}{\bar\lambda_{k-1}} D_f(y^k,x^k)
	-
	A_{k-1} D_f(x^{k-1},x^k)
	\\
	&\qquad{\scriptstyle\mathrm{(a)}}\le
    -\tfrac{\beta_\nu A_k}{c_\nu}U_k
    -\beta_\nu A_{k-1}V_k
	\\
	&\qquad{\scriptstyle\mathrm{(b)}}\le
	-\tfrac{2C_\nu A_k}{L_\nu^{1/\nu}}U_k
	-\tfrac{C_\nu A_{k-1}}{L_\nu^{1/\nu}}V_k
	\\
	&\qquad{\scriptstyle\mathrm{(c)}}\le
	-\tfrac{C_\nu A_{k-1}}{L_\nu^{1/\nu}}(U_k+V_k+U_{k-1})
	-\tfrac{C_\nu A_k}{L_\nu^{1/\nu}}U_k
	+\tfrac{C_\nu A_{k-1}}{L_\nu^{1/\nu}}U_{k-1}
	\\
	&\qquad{\scriptstyle\mathrm{(d)}}\le
	-\tfrac{c_\nu A_{k-1}}{L_\nu^{1/\nu}}W_k
	-\tfrac{C_\nu A_k}{L_\nu^{1/\nu}}U_k
	+\tfrac{C_\nu A_{k-1}}{L_\nu^{1/\nu}}U_{k-1}.
    \end{aligned}
    \numberthis\label{eq:holder-negative-uv}
    \]
	Here (a) uses the co-coercivity lower bound and $a_k/\bar\lambda_{k-1}\ge A_k/c_\nu$; (b) uses the choice of $c_\nu$ satisfying $K_\nu c_\nu\le b_\nu$ and $2K_\nu c_\nu^2\le b_\nu$; (c) uses $A_{k-1}\le A_k$; and (d) uses \eqref{eq:holder-w-triangle} and $C_\nu=c_\nu K_\nu$.
	Adding back the additional nonpositive Bregman terms therefore gives, for each $k\ge1$,
	\[
	\begin{aligned}
	&-\tfrac{a_k}{\bar\lambda_{k-1}} D_f(y^k,x^k)
-
\tfrac{a_k(1-\bar\lambda_{k-1})}{\bar\lambda_{k-1}} D_f(x^k,y^k)
-
A_{k-1} D_f(x^{k-1},x^k)
-
	a_k D_f(x,y^k)
	\\
	&\qquad\le
	-
\tfrac{c_\nu A_{k-1}}{L_\nu^{1/\nu}}
W_k
-
\tfrac{C_\nu A_k}{L_\nu^{1/\nu}}
U_k
+
\tfrac{C_\nu A_{k-1}}{L_\nu^{1/\nu}}
U_{k-1},
	\end{aligned}
	\]
	Summing over $k=1,\dots,n-1$ telescopes the last two terms:
	\[
	\textstyle -\sum_{k=1}^{n-1}A_kU_k+\sum_{k=1}^{n-1}A_{k-1}U_{k-1}
=
-A_{n-1}U_{n-1}\le0,
\]
because $U_0=0$.
Dropping this nonpositive remainder and combining with \eqref{eq:expectation} gives \eqref{eq:soda-otb-holder}. For $\nu=1$, we have $K_1=3$ and $b_1=1/2$, so the choice $c_1=1/6$ satisfies $K_1c_1\le b_1$ and $2K_1c_1^2\le b_1$.
\end{proof}

\subsection{H\"older-smooth theorems}

\begin{thm}[Non-accelerated parameterization under H\"older smoothness]
\label{thm:soda-nonacc-holder-uniform}
Let $p>2$, let $q=p/(p-1)$, and let $\nu\in(0,1]$.
Let $x^\star\in\argmin_{x\in\operatorname{dom}h} f(x)$ and $D_\star:=h(x^\star)-\inf h$.
Consider \ref{eq:SODA}.
For every $k=0,\dots,n-1$, choose
\[
\alpha_k = \tfrac{1}{k+1},
\qquad
\bar\alpha_k = \lambda_k = \tfrac{1}{k+2},
\qquad
\bar\lambda_k \le c_\nu \lambda_k,
\qquad
\gamma_k = \eta(k+2),
\]
where $c_\nu>0$ is the constant from \Cref{lem:soda-otb-holder}.
Suppose \Cref{ass:convex,ass:soda-stochastic,ass:holder-smooth,ass:noise} hold with $q=p/(p-1)$.
Assume also that $h$ is $p$-uniformly convex with constant $\mu$ in the sense of \Cref{def:uniform-convexity}.
Define
\[
\kappa_{\rm na}
:=
\min\left\{1,\tfrac{1+(p+1)\nu}{p}\right\}.
\]
Set
\[
\ell_n :=
\begin{cases}
\log(\mathrm e n), & \nu=(p-1)/(p+1),\\
1, & \text{otherwise}.
\end{cases}
\]
Choose
\[
\eta
\asymp_{p,\nu}
\min\left\{
\tfrac{\mu^{(1+\nu)/p}D_\star^{(p-1-\nu)/p}}
{\tau_q^{1+\nu}L_\nu n^{1-\kappa_{\rm na}}\ell_n^{(p-1-\nu)/p}},
\tfrac{\mu^{1/p}D_\star^{(p-1)/p}}
{\sigma_q\,n^{(p-1)/p}}
\right\}.
\]
Then, for every $n\ge1$,
\[
\mathbb E[f(x^{n-1})-f(x^\star)]
=
O_{p,\nu}\!\left(
\tfrac{\tau_q^{1+\nu}L_\nu D_\star^{(1+\nu)/p}}
{\mu^{(1+\nu)/p}n^{\kappa_{\rm na}}}
+
\tfrac{\sigma_q D_\star^{1/p}}{\mu^{1/p}n^{1/p}}
\right),
\]
up to an additional factor $(\log(\mathrm e n))^{(p-1-\nu)/p}$ when $\nu=(p-1)/(p+1)$.
In particular, for $h(x)=\tfrac1p\norm{x-z^0}_p^p$ and $R_p:=\norm{x^\star-z^0}_p$, taking $\mu=2^{2-p}$ gives
\[
\mathbb E[f(x^{n-1})-f(x^\star)]
=
O_{p,\nu}\!\left(
\tfrac{\tau_q^{1+\nu}L_{\nu,p} R_p^{1+\nu}}{n^{\kappa_{\rm na}}}
+
\tfrac{\sigma_q R_p}{n^{1/p}}
\right),
\]
with the same boundary logarithm.
\end{thm}

\begin{proof}
Choose the weights $a_k\equiv1$, so that $A_k=k+1$.
Combining \Cref{lem:oda-regret} with \Cref{lem:soda-otb-holder} and taking expectations gives
\[
\begin{aligned}
n\,\mathbb E[f(x^{n-1})-f(x^\star)]
&\le
\tfrac{D_\star}{\eta}
+
\tfrac{\eta^{q-1}}{q\mu^{q-1}}
\textstyle\sum_{k=0}^{n-1}
\mathbb E\norm{g^k-g^{k-1}}_*^q
\\
&\qquad
-
\tfrac{c_\nu}{L_\nu^{1/\nu}}
\textstyle\sum_{k=1}^{n-1} k\,
\mathbb E\!\left[
\norm{\nabla f(y^k)-\nabla f(y^{k-1})}_*^{r_\nu}
\right].
\end{aligned}
\]
We next split this last sum into stochastic and deterministic-gradient variation terms.
For $k\ge1$, raising \Cref{ass:noise} to the $q$th power and using $(a+b)^q\lesssim_q a^q+b^q$ gives
\[
\mathbb E\norm{g^k-g^{k-1}}_*^q
\lesssim_q
\tau_q^q
\mathbb E\norm{\nabla f(y^k)-\nabla f(y^{k-1})}_*^q
+
\sigma_q^q .
\]
For $k=0$, using the boundary conventions $g^{-1}=0$ and $y^{-1}=x^\star$, \Cref{ass:noise} gives the bound on the startup term
\[
\mathbb E\norm{g^0-g^{-1}}_*^q
\lesssim_q
\tau_q^q
\mathbb E\norm{\nabla f(y^0)-\nabla f(x^\star)}_*^q
+
\sigma_q^q .
\]
Absorbing only $q$-dependent constants, we obtain
\[
\begin{aligned}
n\,\mathbb E[f(x^{n-1})-f(x^\star)]
&\lesssim_q
\tfrac{D_\star}{\eta}
+
\tfrac{\eta^{q-1}\sigma_q^q n}{q\mu^{q-1}}
\\
&\qquad
+
\tfrac{\tau_q^q\eta^{q-1}}{q\mu^{q-1}}
\textstyle \sum_{k=1}^{n-1}
\mathbb E\!\left[
\norm{\nabla f(y^k)-\nabla f(y^{k-1})}_*^q
\right]
\\
&\qquad
-
\tfrac{c_\nu}{L_\nu^{1/\nu}}
\textstyle \sum_{k=1}^{n-1} k\,
\mathbb E\!\left[
\norm{\nabla f(y^k)-\nabla f(y^{k-1})}_*^{r_\nu}
\right]
\\
&\qquad
+
\tfrac{\tau_q^q\eta^{q-1}}{q\mu^{q-1}}
\mathbb E\!\left[
\norm{\nabla f(y^0)-\nabla f(x^\star)}_*^q
\right].
\end{aligned}
\numberthis\label{eq:nonacc-holder-noise-split}
\]
The last term of \eqref{eq:nonacc-holder-noise-split} is dominated by the first term under the stated choice of $\eta$, since uniform convexity gives $\norm{y^0-x^\star}=O_p((D_\star/\mu)^{1/p})$, and \Cref{ass:holder-smooth} gives
\[
\norm{\nabla f(y^0)-\nabla f(x^\star)}_*^q
\le
L_\nu^q\norm{y^0-x^\star}^{\nu q}
=
O_p\!\left(L_\nu^q(D_\star/\mu)^{\nu q/p}\right).
\]
Hence the last term of \eqref{eq:nonacc-holder-noise-split} is at most
\[
O_p\!\left(
\tfrac{\tau_q^q\eta^{q-1}L_\nu^qD_\star^{\nu q/p}}
{\mu^{q-1+\nu q/p}}
\right).
\]
To compare with the first term $D_\star/\eta$, it is enough that
\[
\tfrac{\tau_q^q\eta^{q-1}L_\nu^qD_\star^{\nu q/p}}
{\mu^{q-1+\nu q/p}}
\lesssim_p
\tfrac{D_\star}{\eta}
\quad\Longleftrightarrow\quad
\eta^q
\lesssim_p
\tfrac{\mu^{q-1+\nu q/p}D_\star^{1-\nu q/p}}
{\tau_q^qL_\nu^q}.
\]
Taking $q$th roots, and using $q=p/(p-1)$, this becomes
\[
\eta
\lesssim_p
\tfrac{\mu^{(1+\nu)/p}D_\star^{(p-1-\nu)/p}}{\tau_q L_\nu},
\]
which is enforced by the deterministic part of the chosen stepsize since $\tau_q\ge1$.

To deal with the two remaining terms of \eqref{eq:nonacc-holder-noise-split}, for $k\ge1$,
\[
\sup_{s\ge0}\{A s^q-Bs^{r_\nu}\}
\le
C_{p,\nu} A^{r_\nu/(r_\nu-q)}B^{-q/(r_\nu-q)}
\qquad
\text{for all } A,B>0,
\]
where $C_{p,\nu}$ is a finite constant depending only on $p$ and $\nu$.
This follows by optimizing over $s$.
Since $r_\nu=(1+\nu)/\nu>q=p/(p-1)$, the supremum is finite and the nonzero maximizer satisfies $s^{r_\nu-q}=qA/(r_\nu B)$.
Applying this bound with
\[
A=\tfrac{\tau_q^q\eta^{q-1}}{\mu^{q-1}},
\qquad
B=\tfrac{k}{L_\nu^{1/\nu}},
\]
yields
\[
\sup_{s\ge0}
\left\{
\tfrac{\tau_q^q\eta^{q-1}}{\mu^{q-1}}s^q
-
\tfrac{k}{L_\nu^{1/\nu}}s^{r_\nu}
\right\}
\le
C_{p,\nu}
\tfrac{\tau_q^{p(1+\nu)/(p-1-\nu)}
L_\nu^{p/(p-1-\nu)}}{\mu^{(1+\nu)/(p-1-\nu)}}
\eta^{(1+\nu)/(p-1-\nu)}
k^{-p\nu/(p-1-\nu)} .
\numberthis\label{eq:nonacc-holder-young-bound}
\]
The exponents in this expression come from
\[
r_\nu-q=\tfrac{p-1-\nu}{\nu(p-1)},
\qquad
\tfrac{r_\nu}{r_\nu-q}=\tfrac{(1+\nu)(p-1)}{p-1-\nu},
\qquad
\tfrac{q}{r_\nu-q}=\tfrac{p\nu}{p-1-\nu}.
\]
To keep the summed bound readable, define
\[
\theta_\nu:=\tfrac{1+\nu}{p-1-\nu},
\qquad
\zeta_\nu:=\tfrac{p\nu}{p-1-\nu},
\qquad
\mathfrak S_{p,\nu,n}:=1+\textstyle\sum_{k=1}^{n-1}k^{-\zeta_\nu},
\qquad
M_\nu:=
\tfrac{L_\nu^{p/(p-1-\nu)}}
{\mu^{(1+\nu)/(p-1-\nu)}},
\]
With this notation, \eqref{eq:nonacc-holder-young-bound} applies pointwise with
\[
s_k:=\norm{\nabla f(y^k)-\nabla f(y^{k-1})}_* .
\]
Combining \eqref{eq:nonacc-holder-young-bound} with the corresponding positive and negative variation terms in \eqref{eq:nonacc-holder-noise-split} gives, for each $k\ge1$,
\[
\begin{aligned}
\tfrac{\tau_q^q\eta^{q-1}}{q\mu^{q-1}}
\mathbb E[s_k^q]
-
\tfrac{c_\nu}{L_\nu^{1/\nu}}k\,
\mathbb E[s_k^{r_\nu}]
\le
O_{p,\nu}\!\left(
\tau_q^{p(1+\nu)/(p-1-\nu)}
M_\nu\eta^{\theta_\nu}k^{-\zeta_\nu}
\right).
\end{aligned}
\]
Thus the positive and negative variation sums in \eqref{eq:nonacc-holder-noise-split} contribute at most the sum of these remainders.
Summing over $k$ and dividing by $n$ gives
\begin{equation}
\label{eq:nonacc-holder-rate-before-balance}
\mathbb E[f(x^{n-1})-f(x^\star)]
\le
O_{p,\nu}\!\left(
\tfrac{D_\star}{\eta n}
+
\tfrac{\eta^{q-1}\sigma_q^q}{\mu^{q-1}}
+
\tau_q^{p(1+\nu)/(p-1-\nu)}M_\nu\eta^{\theta_\nu}\tfrac1n
\mathfrak S_{p,\nu,n}
\right).
\end{equation}
Equating the deterministic and H\"older-variation terms in \eqref{eq:nonacc-holder-rate-before-balance}, and multiplying by $n$, gives
\[
\tfrac{D_\star}{\eta}
\asymp_{p,\nu}
\tau_q^{p(1+\nu)/(p-1-\nu)}
M_\nu\eta^{\theta_\nu}\mathfrak S_{p,\nu,n}
\quad\Longleftrightarrow\quad
\eta^{1+\theta_\nu}
\asymp_{p,\nu}
\tfrac{D_\star}
{\tau_q^{p(1+\nu)/(p-1-\nu)}M_\nu\mathfrak S_{p,\nu,n}}.
\]
Since
\[
1+\theta_\nu=\tfrac{p}{p-1-\nu},
\]
this gives
\[
\eta_{\rm det}
\asymp_{p,\nu}
\tfrac{\mu^{(1+\nu)/p}D_\star^{(p-1-\nu)/p}}
{\tau_q^{1+\nu}L_\nu \mathfrak S_{p,\nu,n}^{(p-1-\nu)/p}},
\]
where the power $\tau_q^{1+\nu}$ comes from
\[
\tfrac{p(1+\nu)}{p-1-\nu}\cdot\tfrac{p-1-\nu}{p}
=1+\nu.
\]
Substituting this value of $\eta_{\rm det}$ into the deterministic term $D_\star/(\eta n)$ in \eqref{eq:nonacc-holder-rate-before-balance} gives
\begin{equation}
\label{eq:nonacc-holder-det-sharp}
O_{p,\nu}\!\left(
\tfrac{\tau_q^{1+\nu}L_\nu D_\star^{(1+\nu)/p}}
{\mu^{(1+\nu)/p}}
\tfrac{\mathfrak S_{p,\nu,n}^{(p-1-\nu)/p}}{n}
\right).
\end{equation}
The cruder $O(n^{-\kappa_{\rm na}})$ rate then follows by applying the standard integral comparison
\[
\mathfrak S_{p,\nu,n}
=1+\textstyle\sum_{k=1}^{n-1}k^{-\zeta_\nu}
\lesssim_{p,\nu}
\begin{cases}
n^{1-\zeta_\nu}, & \zeta_\nu<1,\\
\log(\mathrm e n), & \zeta_\nu=1,\\
1, & \zeta_\nu>1,
\end{cases}
\]
where the boundary $\zeta_\nu=1$ is exactly $\nu=(p-1)/(p+1)$.
Consequently, if $\zeta_\nu<1$, then
\[
\tfrac{\mathfrak S_{p,\nu,n}^{(p-1-\nu)/p}}{n}
\lesssim_{p,\nu}
n^{-\kappa_{\rm na}},
\]
while for $\zeta_\nu>1$ this factor is $O_{p,\nu}(n^{-1})$, and at $\zeta_\nu=1$ it is $n^{-1}(\log(\mathrm e n))^{(p-1-\nu)/p}$.
Thus the deterministic contribution is
\[
O_{p,\nu}\!\left(
\tfrac{\tau_q^{1+\nu}L_\nu D_\star^{(1+\nu)/p}}
{\mu^{(1+\nu)/p}n^{\kappa_{\rm na}}}
\right),
\]
up to the logarithmic factor at the boundary.

Equating the deterministic and stochastic terms in \eqref{eq:nonacc-holder-rate-before-balance} gives
\[
\tfrac{D_\star}{\eta n}
\asymp_p
\tfrac{\eta^{q-1}\sigma_q^q}{\mu^{q-1}},
\qquad\text{hence}\qquad
\eta^q
\asymp_p
\tfrac{\mu^{q-1}D_\star}{\sigma_q^q n}.
\]
Using $q=p/(p-1)$, this yields
\[
\eta_{\rm stoch}
\asymp_p
\tfrac{\mu^{1/p}D_\star^{(p-1)/p}}
{\sigma_q n^{(p-1)/p}},
\]
Plugging this into the deterministic term of \eqref{eq:nonacc-holder-rate-before-balance} gives the stochastic rate $\sigma_qD_\star^{1/p}/(\mu^{1/p}n^{1/p})$.
\end{proof}

\begin{cor}[Near $p=\infty$ non-accelerated parameterization]
\label{cor:soda-near-linf}
Let $d$ denote the ambient dimension, or the effective rank in the Schatten case, and set
\[
p_d:=\max\{3,\lceil \log(\mathrm e d)\rceil\},
\qquad
q_d:=\tfrac{p_d}{p_d-1}.
\]
Let $x^\star\in\argmin_{x\in\operatorname{dom}h} f(x)$ and $R_\infty:=\norm{x^\star-z^0}_\infty$.
Run \ref{eq:SODA} with the finite $p_d$ geometry $h(x)=\tfrac1{p_d}\norm{x-z^0}_{p_d}^{p_d}$ and the non-accelerated parameters from \Cref{thm:soda-nonacc-holder-uniform} with $\nu=1$.
Suppose the assumptions of \Cref{thm:soda-nonacc-holder-uniform}, except for \Cref{ass:holder-smooth}, hold for this choice of $p_d$, $q_d$, and $h$.
Assume instead that $f$ is $L_\infty$-smooth with respect to $\norm{\cdot}_\infty$.
Then
\[
\mathbb E[f(x^{n-1})-f(x^\star)]
\le
C_{\tau_{q_d}}
\left(
\tfrac{L_\infty R_\infty^2
\min\{\log(\mathrm e n),\log(\mathrm e^3 d)\}}{n}
+
\tfrac{\sigma_{q_d}R_\infty}{n^{1/p_d}}
\right).
\]
\end{cor}

\begin{proof}
We apply \Cref{thm:soda-nonacc-holder-uniform} with $\nu=1$ and $p=p_d$.
The radius and smoothness constant can be controlled by their $\ell_\infty$ analogues.
The radius satisfies
\[
\norm{x^\star-z^0}_{p_d}
\le d^{1/p_d}\norm{x^\star-z^0}_\infty
\le \mathrm e R_\infty .
\]
The smoothness conversion uses the dual-norm inequality $\norm{\cdot}_{q_d}\le\norm{\cdot}_1$.
Therefore
\[
\norm{\nabla f(x)-\nabla f(y)}_{q_d}
\le
\norm{\nabla f(x)-\nabla f(y)}_1
\le
L_\infty\norm{x-y}_\infty
\le
L_\infty\norm{x-y}_{p_d}.
\]
Thus the finite-$p_d$ theorem applies with $R_{p_d}\le \mathrm e R_\infty$ and $L_{p_d}\le L_\infty$.
By the explicit deterministic contribution in \eqref{eq:nonacc-holder-det-sharp}, with $\nu=1$ and $p=p_d$, the constants depending on $q_d\in[1,3/2]$ and $\mu=2^{2-p_d}$ are uniformly bounded.
The remaining explicit $p_d$ factor is
\[
\mathfrak S_{p_d,1,n}^{(p_d-2)/p_d},
\qquad
\textstyle \mathfrak S_{p_d,1,n}:=1+\sum_{k=1}^{n-1} k^{-p_d/(p_d-2)} .
\]
Since $p_d/(p_d-2)=1+2/(p_d-2)$, an integral test gives
\[
\textstyle \sum_{k=1}^{n-1} k^{-p_d/(p_d-2)}
\le
1+\int_1^\infty x^{-1-2/(p_d-2)}\,dx
=
1+\tfrac{p_d-2}{2}.
\]
The same sum is also bounded by the harmonic series, $\sum_{k=1}^{n-1}k^{-1}\le\log(\mathrm e n)$.
Since
\[
\mathfrak S_{p_d,1,n}
\le
1+\min\{\log(\mathrm e n),1+\tfrac{p_d-2}{2}\},
\]
we have
\[
\mathfrak S_{p_d,1,n}^{(p_d-2)/p_d}
\lesssim
\min\{\log(\mathrm e n),p_d\}
\lesssim
\min\{\log(\mathrm e n),\log(\mathrm e^3d)\}.
\]
\end{proof}
\begin{remark}
This corollary uses a finite approximation to $\infty$-norm.
This matches practice for Schatten-$\infty$, where the matrix sign operation is typically not computed exactly, but approximated by an iterative solver such as Newton--Schulz.
\end{remark}
We remark that it is also possible to obtain a similar result for the accelerated parameterization in \Cref{thm:soda-acc-holder-uniform}.
However, we do not pursue this here as it does not lead to an improvement in the asymptotic rate, since the $O(1/n)$ term is optimal for $p=\infty$ in the smooth convex deterministic setting \citep[Cor.~1]{guzman2018lower}.

\begin{thm}[Accelerated parameterization under H\"older smoothness]
\label{thm:soda-acc-holder-uniform}
Let $p>2$, let $q=p/(p-1)$, and let $\nu\in(0,1]$.
Let $x^\star\in\argmin_{x\in\operatorname{dom}h} f(x)$ and $D_\star:=h(x^\star)-\inf h$.
Consider \ref{eq:SODA}.
For every $k=0,\dots,n-1$, choose
\[
\alpha_k = \tfrac{2}{k+2},
\qquad
\bar\alpha_k = \lambda_k = \tfrac{2}{k+3},
\qquad
\bar\lambda_k \le c_\nu \lambda_k,
\qquad
\gamma_k = \eta \tfrac{(k+2)(k+3)}{2},
\]
where $c_\nu>0$ is the constant from \Cref{lem:soda-otb-holder}.
Suppose \Cref{ass:convex,ass:soda-stochastic,ass:holder-smooth,ass:noise} hold with $q=p/(p-1)$.
Assume also that $h$ is $p$-uniformly convex with constant $\mu$ in the sense of \Cref{def:uniform-convexity}.
Define
\[
\kappa_{\rm acc}:=\tfrac{1+(p+1)\nu}{p}.
\]
Choose
\[
\eta
\asymp_{p,\nu}
\min\left\{
\tfrac{\mu^{(1+\nu)/p}D_\star^{(p-1-\nu)/p}}
{\tau_q^{1+\nu}L_\nu n^{2-\kappa_{\rm acc}}},
\tfrac{\mu^{1/p}D_\star^{(p-1)/p}}
{\sigma_q\,n^{(2p-1)/p}}
\right\}.
\]
Then, for every $n\ge1$,
\[
\mathbb E[f(x^{n-1})-f(x^\star)]
=
O_{p,\nu}\!\left(
\tfrac{\tau_q^{1+\nu}L_\nu D_\star^{(1+\nu)/p}}
{\mu^{(1+\nu)/p}n^{\kappa_{\rm acc}}}
+
\tfrac{\sigma_q D_\star^{1/p}}{\mu^{1/p}n^{1/p}}
\right).
\]
In particular, for $h(x)=\tfrac1p\norm{x-z^0}_p^p$ and $R_p:=\norm{x^\star-z^0}_p$, taking $\mu=2^{2-p}$ gives
\[
\mathbb E[f(x^{n-1})-f(x^\star)]
=
O_{p,\nu}\!\left(
\tfrac{\tau_q^{1+\nu}L_{\nu,p} R_p^{1+\nu}}{n^{\kappa_{\rm acc}}}
+
\tfrac{\sigma_q R_p}{n^{1/p}}
\right).
\]
\end{thm}

\begin{proof}
Choose $a_k=k+1$, so that $A_k=(k+1)(k+2)/2$.
Combining \Cref{lem:oda-regret} with \Cref{lem:soda-otb-holder} and taking expectations gives
\[
\begin{aligned}
A_{n-1}\mathbb E[f(x^{n-1})-f(x^\star)]
&\le
\tfrac{D_\star}{\eta}
+
\tfrac{\eta^{q-1}}{q\mu^{q-1}}
\textstyle\sum_{k=0}^{n-1}(k+1)^q
\mathbb E\norm{g^k-g^{k-1}}_*^q
\\
&\qquad
-
\tfrac{c_\nu}{L_\nu^{1/\nu}}
\textstyle \sum_{k=1}^{n-1} A_{k-1}
\mathbb E\!\left[
\norm{\nabla f(y^k)-\nabla f(y^{k-1})}_*^{r_\nu}
\right]
\end{aligned}
\]
We next split the regret noise term.
For $k\ge1$, raising \Cref{ass:noise} to the $q$th power and using $(a+b)^q\lesssim_q a^q+b^q$ gives
\[
\mathbb E\norm{g^k-g^{k-1}}_*^q
\lesssim_q
\tau_q^q
\mathbb E\norm{\nabla f(y^k)-\nabla f(y^{k-1})}_*^q
+
\sigma_q^q .
\]
For $k=0$, using the boundary conventions $g^{-1}=0$ and $y^{-1}=x^\star$, \Cref{ass:noise} gives
\[
\mathbb E\norm{g^0-g^{-1}}_*^q
\lesssim_q
\tau_q^q
\mathbb E\norm{\nabla f(y^0)-\nabla f(x^\star)}_*^q
+
\sigma_q^q .
\]
Since $\sum_{k=0}^{n-1}(k+1)^q\lesssim_q n^{q+1}$, absorbing only $q$-dependent constants gives
\[
\begin{aligned}
A_{n-1}\mathbb E[f(x^{n-1})-f(x^\star)]
&\lesssim_q
\tfrac{D_\star}{\eta}
+
\tfrac{\eta^{q-1}\sigma_q^q n^{q+1}}{q\mu^{q-1}}
\\
&\qquad
+
\tfrac{\tau_q^q\eta^{q-1}}{q\mu^{q-1}}
\textstyle \sum_{k=1}^{n-1} (k+1)^q
\mathbb E\!\left[
\norm{\nabla f(y^k)-\nabla f(y^{k-1})}_*^q
\right]
\\
&\qquad
-
\tfrac{c_\nu}{L_\nu^{1/\nu}}
\textstyle \sum_{k=1}^{n-1} A_{k-1}
\mathbb E\!\left[
\norm{\nabla f(y^k)-\nabla f(y^{k-1})}_*^{r_\nu}
\right]
\\
&\qquad
+
\tfrac{\tau_q^q\eta^{q-1}}{q\mu^{q-1}}
\mathbb E\!\left[
\norm{\nabla f(y^0)-\nabla f(x^\star)}_*^q
\right].
\end{aligned}
\numberthis\label{eq:acc-holder-noise-split}
\]
The last term of \eqref{eq:acc-holder-noise-split} is dominated by the first term by the same argument as in the non-accelerated proof, since the deterministic branch of $\eta$ is again at most
\[
\tfrac{\mu^{(1+\nu)/p}D_\star^{(p-1-\nu)/p}}{\tau_q L_\nu}.
\]
It remains to control the positive and negative variation sums.
Since $(k+1)^q\lesssim_q k^q$ and $A_{k-1}\gtrsim k^2$ for $k\ge1$, it is enough to use the same scalar bound as in the non-accelerated proof:
\[
\sup_{s\ge0}\{A s^q-Bs^{r_\nu}\}
\le
C_{p,\nu} A^{r_\nu/(r_\nu-q)}B^{-q/(r_\nu-q)}
\qquad
\text{for all } A,B>0.
\]
Here $C_{p,\nu}$ is a finite constant depending only on $p$ and $\nu$.
Applying this bound with
\[
A=\tfrac{\tau_q^q\eta^{q-1}}{\mu^{q-1}}k^q,
\qquad
B=\tfrac{k^2}{L_\nu^{1/\nu}},
\]
yields
\[
\sup_{s\ge0}
\left\{
\tfrac{\tau_q^q\eta^{q-1}}{\mu^{q-1}}k^q s^q
-
\tfrac{k^2}{L_\nu^{1/\nu}}s^{r_\nu}
\right\}
\le
C_{p,\nu}
\tfrac{\tau_q^{p(1+\nu)/(p-1-\nu)}
L_\nu^{p/(p-1-\nu)}}{\mu^{(1+\nu)/(p-1-\nu)}}
\eta^{(1+\nu)/(p-1-\nu)}
k^{p(1-\nu)/(p-1-\nu)} .
\numberthis\label{eq:acc-holder-young-bound}
\]
The exponent of $k$ in \eqref{eq:acc-holder-young-bound} comes from
\[
q\tfrac{r_\nu}{r_\nu-q}
-
2\tfrac{q}{r_\nu-q}
=
\tfrac{p(1-\nu)}{p-1-\nu}.
\]
Define, as before,
\[
\theta_\nu:=\tfrac{1+\nu}{p-1-\nu},
\qquad
\xi_\nu:=\tfrac{p(1-\nu)}{p-1-\nu},
\qquad
M_\nu:=
\tfrac{L_\nu^{p/(p-1-\nu)}}
{\mu^{(1+\nu)/(p-1-\nu)}} .
\]
With this notation, \eqref{eq:acc-holder-young-bound} applies pointwise with
\[
s_k:=\norm{\nabla f(y^k)-\nabla f(y^{k-1})}_* .
\]
Combining \eqref{eq:acc-holder-young-bound} with $(k+1)^q\lesssim_q k^q$ and $A_{k-1}\gtrsim k^2$ gives, for each $k\ge1$,
\[
\begin{aligned}
\tfrac{\tau_q^q\eta^{q-1}}{q\mu^{q-1}}(k+1)^q
\mathbb E[s_k^q]
-
\tfrac{c_\nu}{L_\nu^{1/\nu}}A_{k-1}
\mathbb E[s_k^{r_\nu}]
\le
O_{p,\nu}\!\left(
\tau_q^{p(1+\nu)/(p-1-\nu)}
M_\nu\eta^{\theta_\nu}k^{\xi_\nu}
\right).
\end{aligned}
\]
Thus the positive and negative variation sums in \eqref{eq:acc-holder-noise-split} contribute at most the sum of these remainders.
Since
\[
\textstyle\sum_{k=1}^{n-1} k^{\xi_\nu}
\lesssim_{p,\nu}
n^{1+\xi_\nu},
\qquad
A_{n-1}=\Theta(n^2),
\]
dividing by $A_{n-1}$ gives
\begin{equation}
\label{eq:acc-holder-rate-before-balance}
\begin{aligned}
\mathbb E[f(x^{n-1})-f(x^\star)]
\le
O_{p,\nu}\!\left(
\tfrac{D_\star}{\eta n^2}
+
\tau_q^{p(1+\nu)/(p-1-\nu)}
M_\nu\eta^{\theta_\nu}n^{\xi_\nu-1}
+
\tfrac{\eta^{q-1}\sigma_q^q n^{q-1}}{\mu^{q-1}}
\right).
\end{aligned}
\end{equation}
Equating the deterministic and H\"older-variation terms in \eqref{eq:acc-holder-rate-before-balance} gives
\[
\tfrac{D_\star}{\eta n^2}
\asymp_{p,\nu}
\tau_q^{p(1+\nu)/(p-1-\nu)}
M_\nu\eta^{\theta_\nu}n^{\xi_\nu-1}
\quad\Longleftrightarrow\quad
\eta^{1+\theta_\nu}
\asymp_{p,\nu}
\tfrac{D_\star}
{\tau_q^{p(1+\nu)/(p-1-\nu)}M_\nu n^{1+\xi_\nu}}.
\]
Since
\[
1+\theta_\nu=\tfrac{p}{p-1-\nu},
\qquad
(1+\xi_\nu)\tfrac{p-1-\nu}{p}
=2-\kappa_{\rm acc},
\]
this yields
\[
\eta_{\rm det}
\asymp_{p,\nu}
\tfrac{\mu^{(1+\nu)/p}D_\star^{(p-1-\nu)/p}}
{\tau_q^{1+\nu}L_\nu n^{2-\kappa_{\rm acc}}},
\qquad
\kappa_{\rm acc}=\tfrac{1+(p+1)\nu}{p}.
\]
Substituting this value of $\eta_{\rm det}$ into $D_\star/(\eta n^2)$ gives the deterministic rate in the theorem.

Equating the deterministic and stochastic terms in \eqref{eq:acc-holder-rate-before-balance} gives
\[
\tfrac{D_\star}{\eta n^2}
\asymp_p
\tfrac{\eta^{q-1}\sigma_q^q n^{q-1}}{\mu^{q-1}},
\qquad\text{hence}\qquad
\eta^q
\asymp_p
\tfrac{\mu^{q-1}D_\star}{\sigma_q^q n^{q+1}}.
\]
Using $q=p/(p-1)$, this gives
\[
\eta_{\rm stoch}
\asymp_p
\tfrac{\mu^{1/p}D_\star^{(p-1)/p}}
{\sigma_q n^{(2p-1)/p}},
\]
and substitution into $D_\star/(\eta n^2)$ gives the stochastic rate $\sigma_qD_\star^{1/p}/(\mu^{1/p}n^{1/p})$.
\end{proof}

\begin{cor}[Accelerated oracle complexity with less heavy-tailed noise]
\label{cor:acc-minibatch-finite-p}
Let $p\ge2$, set $q=p/(p-1)$, and suppose $p\ge p_r$.
Suppose \Cref{ass:mean-zero-r-moment} holds.
Run the accelerated parameterization in \Cref{cor:soda-acc-uniform} with gradient feedback $g^k=\bar g_B(y^k)$.
Then \Cref{cor:soda-acc-uniform} applies with its noise parameter $\sigma_q$ replaced by $d^{1/q-1/r}\sigma_rB^{-1/p_r}$, giving
\[
\mathbb E[f(x^{n-1})-f(x^\star)]
\lesssim
\tfrac{A_p}{n^{1+2/p}}
+
\tfrac{S_{r,p}}{B^{1/p_r}n^{1/p}},
\]
where $R_p:=\norm{x^\star-z^0}_p$, $A_p=L_pR_p^2$, and $S_{r,p}:=\sigma_r d^{1/q-1/r}R_p$.
Hence accuracy $\epsilon$ is reached by taking
\[
n\asymp
\left(\tfrac{A_p}{\epsilon}\right)^{p/(p+2)},
\qquad
B\asymp
\left(\tfrac{S_{r,p}}{\epsilon}\right)^{p_r}
\left(\tfrac{\epsilon}{A_p}\right)^{p_r/(p+2)},
\]
and therefore with total stochastic gradient budget
\[
N=nB
=
O\!\left(
\left(\tfrac{A_p}{\epsilon}\right)^{p/(p+2)}
+
\left(\tfrac{S_{r,p}}{\epsilon}\right)^{p_r}
\left(\tfrac{A_p}{\epsilon}\right)^{(p-p_r)/(p+2)}
\right).
\]
Equivalently, for a fixed oracle budget $N=nB$,
\[
\epsilon^{\rm acc}_p(n,N)
\lesssim
\tfrac{A_p}{n^{1+2/p}}
+
\tfrac{S_{r,p}}{N^{1/p_r}}\,n^{1/p_r-1/p},
\qquad p\ge p_r .
\]
\end{cor}
\begin{proof}
The mini-batch moment bound follows from the same von Bahr--Esseen and norm comparison argument used in \Cref{prop:minibatch-finite-p} stating that the averaged feedback has effective $q$-moment scale $d^{1/q-1/r}\sigma_rB^{-1/p_r}$.
Substituting this scale into \Cref{cor:soda-acc-uniform} gives the above rate.
The stated choices of $n$ and $B$ make the deterministic and stochastic terms at most constants times $\epsilon$, and the total stochastic gradient budget is $N=nB$.
Substituting $B=N/n$ gives the form where $N$ is fixed.
\end{proof}

\paragraph{Accelerated batch size rule for a fixed geometry}
The fixed oracle budget result in \Cref{cor:acc-minibatch-finite-p} also gives an accelerated analogue of the batch size rule in \eqref{eq:bs-scaling-rule}.
Balancing the smooth and stochastic terms in
\[
\epsilon^{\rm acc}_p(n,N)
\lesssim
\tfrac{A_p}{n^{1+2/p}}
+
\tfrac{S_{r,p}}{N^{1/p_r}}\,n^{1/p_r-1/p}
\]
gives
\[
\begin{aligned}
n_p^{\rm acc}
&\asymp
\left(\tfrac{A_p}{S_{r,p}}\right)^{1/(1+1/p_r+1/p)}
N^{\frac{1/p_r}{1+1/p_r+1/p}},
\\
B_p^{\rm acc}=\tfrac{N}{n_p^{\rm acc}}
&\asymp
\left(\tfrac{S_{r,p}}{A_p}\right)^{1/(1+1/p_r+1/p)}
N^{\frac{1+1/p}{1+1/p_r+1/p}}.
\end{aligned}
\]
At $p=p_r$, the stochastic term is independent of $n$ after substituting $B=N/n$, so this balance gives the largest batch size, or equivalently the smallest iteration count, that remains optimal up to constants.
For bounded variance, $p_r=2$, the batch exponent is
\[
B_p^{\rm acc}
\asymp
N^{\beta_{\rm acc}(p)},
\qquad
\beta_{\rm acc}(p)
=
\tfrac{2(p+1)}{3p+2},
\qquad
\beta_{\rm acc}(2)=\tfrac34,
\quad
\lim_{p\rightarrow\infty}\beta_{\rm acc}(p)=\tfrac23,
\]
which we plot in \Cref{fig:accelerated-batchsize-scaling}.
Thus, the accelerated parameterization allows taking larger batch sizes than the non-accelerated version (\textit{c.f.} \Cref{tab:bounded-var-batchsize}).
\begin{figure}[H]
\centering
\includegraphics[width=0.45\linewidth]{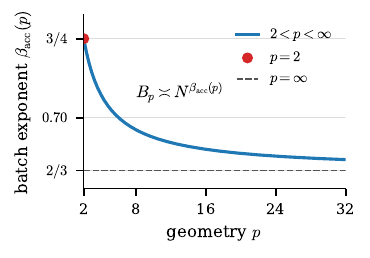}
\caption{Accelerated batch size scaling rule under bounded variance based on
\Cref{cor:acc-minibatch-finite-p}.}
\label{fig:accelerated-batchsize-scaling}
\end{figure}

\paragraph{Accelerated effective stepsize.}
\label{app:accelerated-stepsize-scale}

For the accelerated choice in \Cref{cor:soda-acc-uniform}, the corresponding horizon-free stepsize is obtained by replacing the horizon $n$ in $\eta$ by the current iteration $k$:
\[
\eta_k
\asymp
\min\left\{
C_{\rm det}k^{-(p-2)/p},
C_{\rm stoch}k^{-(2p-1)/p}
\right\}.
\]
Since $\gamma_k\asymp\eta_k k^2$, this gives
\[
\gamma_k^{q-1}
\asymp
\min\left\{
C_{\rm det}^{1/(p-1)}k^{(p+2)/(p(p-1))},
C_{\rm stoch}^{1/(p-1)}k^{1/(p(p-1))}
\right\}.
\]
The effective scale $\lambda_k\gamma_k^{q-1}$, with $\lambda_k\asymp k^{-1}$, therefore has endpoints
\[
p=2:\qquad
\lambda_k\gamma_k^{q-1}
\asymp
\min\left\{
C_{\rm det}k,
C_{\rm stoch}k^{-1/2}
\right\},
\qquad
p\rightarrow\infty:\qquad
\lambda_k\gamma_k^{q-1}
\asymp
k^{-1}.
\]